\documentclass[final]{colt2023} 

\usepackage{enumitem}

\newcommand{\KL}{\mathrm{KL}}
\newcommand{\kl}{\mathrm{kl}}
\newcommand{\alt}{\mathrm{Alt}}

\newtheorem{assumption}{Assumption}

\title[On the Existence of a Complexity in Fixed Budget Bandit Identification]{On the Existence of a Complexity in Fixed Budget Bandit Identification}
\usepackage{times}

\coltauthor{%
 \Name{Rémy Degenne} \Email{remy.degenne@inria.fr}\\
 \addr Univ. Lille, Inria, CNRS, Centrale Lille, UMR 9189-CRIStAL, F-59000 Lille, France
}

\begin{document}

\maketitle

\begin{abstract}%
In fixed budget bandit identification, an algorithm sequentially observes samples from several distributions up to a given final time.
It then answers a query about the set of distributions. A good algorithm will have a small probability of error.
While that probability decreases exponentially with the final time, the best attainable rate is not known precisely for most identification tasks.
We show that if a fixed budget task admits a complexity, defined as a lower bound on the probability of error which is attained by the same algorithm on all bandit problems, then that complexity is determined by the best non-adaptive sampling procedure for that problem.
We show that there is no such complexity for several fixed budget identification tasks including Bernoulli best arm identification with two arms: there is no single algorithm that attains everywhere the best possible rate.
\end{abstract}


\section{Introduction}

A multi-armed bandit is a model of a sequential interaction between an algorithm and its environment. The bandit is described by a finite number of probability distributions (called arms) $\nu_1, \ldots, \nu_K$ with finite means. At every discrete step $t \in \mathbb{N}$, the algorithm chooses one arm $k_t$ and observes a sample $X_t^{k_t}$ from the distribution $\nu_{k_t}$.
The bandit model was introduced to study clinical trials, but has found many applications in recommender systems and online advertisement.

Most of the bandit literature is concerned with the design of algorithms that maximize the expected sum of the samples gathered by the algorithm, which in this case represent rewards accrued by choosing the arms.
See \citep{bubeck2012regret,lattimore2020bandit} for extensive surveys.
We are on the other hand interested in the \emph{identification} setting.
We also consider a set $\mathcal D$ of tuples of real probability distributions (we call such a tuple a \emph{bandit problem}), but we additionally define a finite answer set $\mathcal I$, and a function $i^\star : \mathcal D \to \mathcal I$, called the correct answer function.
We call $(\mathcal D, \mathcal I, i^\star)$ an \emph{identification task}.
An identification algorithm will sequentially observe samples from the unknown distributions $(\nu_1, \ldots, \nu_K) \in \mathcal D$ until a time $\tau$ at which it stops and returns an answer.
Its goal is to return the correct answer with high probability.
At each successive discrete time $t \ge 1$ until a stopping time $\tau$, the algorithm chooses an arm $k_t$ based on previous observations and it observes $X_t^{k_t} \sim \nu_{k_t}$.
At $\tau$, the algorithm returns an answer $\hat{i}_\tau \in \mathcal I$. We say that the answer is correct if $\hat{i}_\tau = i^\star(\nu)$, and that the algorithm makes an error otherwise. We denote by $p_{\nu, \tau}(\mathcal A)$ the probability of error of algorithm $\mathcal A$ on problem $\nu$, that is $p_{\nu, \tau}(\mathcal A) := \mathbb{P}_{\nu, \mathcal A}(\hat{i}_\tau \ne i^\star(\nu))$ (we index the probability by the problem and the algorithm).
The bandit identification problem has mainly been studied in the two following ways:
\begin{itemize}[noitemsep]
  \item \emph{Fixed confidence}: the stopping time $\tau$ is a part of the algorithm design, and we want to find an algorithm $\mathcal A$ with minimal $\mathbb{E}_\mu[\tau]$ under the constraint that for all $\mu \in \mathcal D$, $p_{\mu, \tau}(\mathcal A) \le \delta$ for a known $\delta > 0$.
  \item \emph{Fixed budget}: the stopping time is set to a value $T \in \mathbb{N}$ known in advance, and we are looking for an algorithm $\mathcal A$ with minimal $p_{\mu, T}(\mathcal A)$ for all $\mu \in \mathcal D$.
\end{itemize}

\paragraph{Detailed example: best arm identification}

The bandit identification framework include diverse queries about the distribution, the most popular of which is best arm identification (BAI, \cite{even2006action,bubeck2009pure,audibert2010best,gabillon2012best,karnin2013almost}). Here the goal of the algorithm is to find the arm with highest mean.

Suppose that we know that the distributions of the arms are Bernoulli, but with unknown means: this is encoded in the set of tuples of distributions $\mathcal D = \{(\nu_1, \ldots, \nu_K) \mid \forall k \in [K], \exists \mu_k \in (0,1), \nu_k = \mathcal B(\mu_k)\}$, where $\mathcal B(\mu_k)$ is the Bernoulli distribution with mean $\mu_k$. In that example, the tuple of distributions $\nu$ is uniquely described by the tuple of means $\mu$ and we will talk indifferently about $\nu$ and $\mu$.

We want to find the arm with highest mean, hence the set of answers is $\mathcal I = \{1, \ldots, K\}$. The correct answer function $i^* : \mathcal D \to \mathcal I$ is $i^\star(\mu) = \arg\max_k \mu_k$. To ensure that $i^*$ is a function, with a unique value in $\mathcal I$, we need to restrict $\mathcal D$ to the tuples $\mu$ such that the argmax is unique.

In fixed budget identification, an algorithm would sample an arm at each time until time $T$, then return $\hat{i}_\tau \in [K]$, the arm which it thinks is the one with highest mean. That answer would be correct if $\hat{i}_\tau = i^\star(\nu) = \arg\max_k \mu_k$ and would make a mistake otherwise

\paragraph{Other examples of identification tasks}
\label{par:examples}

Identification is more general than BAI, and we could seek the answer to other queries

\begin{itemize}[noitemsep]
  \item Thresholding Bandits \citep{locatelli2016optimal}: the algorithm returns for all arms whether its mean is below or above a given threshold, and is correct only if all signs are correct. The answer set is $\mathcal I = \{-, +\}^K$.
  
  \item Positivity: the goal of the algorithm is to determine whether all arms have means above a threshold, or if at least one has mean below. The answer set is $\mathcal I = \{\text{all above}, \text{exists below}\}$. It was introduced in \citep{kaufmann2018sequential} as a step towards identification of the best play in two player min-max games, but can also model the task of verifying if all components of a system meet minimal performance thresholds. See also \citep{degenne2019pure}.
\end{itemize}

These two examples vary the answer set and function, $\mathcal I$ and $i^\star$. Variants of these tasks can also be obtained by choosing different sets of distributions $\mathcal D$. For example, the distributions could be Gaussian with same variance and a mean vector result of the product of a known matrix and an unknown low dimensional parameter vector, as in linear bandits. These so-called \emph{structured} settings are the subject of a lot of recent attention in the fixed budget literature \citep{azizi2021fixed,alieva2021robust,yangminimax,cheshire2021problem}.
Our approach of fixed budget identification is frequentist, but a bayesian goal could also be studied, as in \citep{atsidakou2022bayesian}.

\paragraph{Assumptions on the identification problem}
\label{par:assumptions_on_the_identification_problem}

We do not consider all possible identification problems, but restrict our attention to queries about the means of parametric distributions. We suppose that for each arm $k \in [K]$, the set of possible distributions is a subset of a one-parameter canonical exponential family.
For example, all arms may have Gaussian distributions with known variance but unknown mean, or Bernoulli distributions with means in $(0,1)$.
Exponential families is the setting for which fixed confidence is best understood. Bandit identification is of course interesting beyond that model. However the goal of this paper is to show mostly negative results, showing that fixed budget is not as simple as fixed confidence, even in that very simple parametric model.

For such exponential families, the distribution of each arm can be uniquely described by its mean, we identify means and distributions everywhere in the remainder of the paper. We will talk about some bandit problem $\mu \in \mathcal D$ and also denote its mean vector by $\mu$.
The mean of each arm $k \in [K]$ belongs to an open interval $\mathcal M_k$. For any set $S$, let $\mathrm{cl}(S)$ be its closure and $\mathrm{int}(S)$ be its interior.
The empirical mean $\hat{\mu}_{T,k} \in \mathrm{cl}(\mathcal M_k)$ of an arm $k$ is the maximum likelihood estimator for the mean $\mu_k$ and we can have concentration results for that estimator.

Finally, we need to introduce an assumption to make sure that every $\mu \in \mathcal D$ has a well defined correct answer which can reliably be found if we observe enough samples of every arm.

\begin{assumption}\label{asm:answer_nhds}
For all $i \in \mathcal I$, $\mathcal D_i := \{\mu \in \mathcal D \mid i^\star(\mu) = i\}$ is open and $\mathcal D_i = \mathrm{int}(\mathrm{cl}(\mathcal D_i))$.
The union $\bigcup_{i \in \mathcal I} \mathrm{cl}(\mathcal D_i)$ contains all tuples of distributions in the exponential family. Finally, $\mathcal D = \bigcup_{i \in \mathcal I} \mathcal D_i$
\end{assumption}

$\mathcal D_i = \mathrm{int}(\mathrm{cl}(\mathcal D_i))$ ensures that if all problems in a neighborhood of $\mu \in \mathcal D$ have the same answer $i$, then $i^\star(\mu) = i$ as well.
The condition on $\bigcup_{i \in \mathcal I} \mathrm{cl}(\mathcal D_i)$ ensures that the empirical mean of the arms will always be in the closure of $\mathcal D$.
We then extend $i^\star$ beyond $\mathcal D$, to all tuples in $\mathrm{cl}(\mathcal M_1) \times \ldots \times \mathrm{cl}(\mathcal M_K)$, by giving it an arbitrary value outside of $\mathcal D$. We can then define the \emph{empirical correct answer} $i^\star(\hat{\mu}_T)$.
Informally, we required that $\mathcal D$ contains all tuples of distributions for which the correct answer $i^\star$ is unique.
In thresholding bandits $\mathcal D$ contains all tuples for which all arms have means not equal to the threshold.
Everywhere in the paper $\mathcal D$ will satisfy that assumption, even if not explicitly mentioned. For example, if we write that in a BAI task $\mathcal D$ contains Gaussian distributions with variance 1, we mean all tuples such that there is a unique arm with highest mean.

\subsection{Fixed confidence bandit identification}
\label{sub:fixed_confidence_bandit_identification}

Fixed confidence identification is now well understood in the asymptotic regime, when $\delta \to 0$. 
Let's now describe one central facet of asymptotic fixed confidence identification: the existence of a complexity.
To that end we will consider two classes of algorithms.
The first class contains $\delta$-correct algorithms. Denote it $\mathcal C^\delta$. An algorithm is said to be $\delta$-correct on $\mathcal D$ if for all $\mu \in \mathcal D$, $p_{\mu, \tau} \le \delta$.

\cite{garivier2016optimal} showed that there exists a function $H_{\mathcal C^{\delta}} : \mathcal D \to \mathbb{R}$ such that any $\delta$-correct algorithm satisfies, for all $\mu \in \mathcal D$,
\begin{align*}
\liminf_{\delta \to 0} \mathbb{E}_\mu[\tau]/\log(1/\delta) \ge H_{\mathcal C^\delta}(\mu) \: .
\end{align*}
They introduced the Track-and-Stop algorithm (TnS), which is $\delta$-correct and satisfies for all $\mu \in \mathcal D$
\begin{align*}
\limsup_{\delta \to 0} \mathbb{E}_\mu[\tau]/\log(1/\delta) \le H_{\mathcal C^\delta}(\mu) \: .
\end{align*}
The conclusion from these two facts is that we can meaningfully talk about \emph{the complexity} of identification at $\mu$ for $\delta$-correct algorithms: there is a function $H_{\mathcal C^\delta}$ which is a lower bound on $\liminf_{\delta \to 0} \frac{\mathbb{E}_\mu[\tau]}{\log(1/\delta)}$ for all $\mu \in \mathcal D$ and all algorithms $\mathcal A \in \mathcal C^\delta$, and that bound can be matched on every $\mu$ by the same algorithm in the class (TnS for example, among others \citep{degenne2019non,you2022information}).

The second class of interest contains algorithms which are $\delta$-correct and use static proportions, meaning algorithms which are parametrized by $w \in \triangle_K$ (the simplex) and maintain sampling counts at every time $T \in \mathbb{N}$ close to $w_k T$ for each arm $k \in [K]$, say $\vert N_{T,k} - w_k T \vert \le K$ for all $T,k$. Let us denote that class by $\mathcal C^{sp}$.
For $(\mathcal D, \mathcal I, i^\star)$ satisfying our assumptions, there exist stopping rules and recommendation rules which can make any algorithm using them $\delta$-correct, regardless of the sampling rule \citep{garivier2016optimal}. This shows in particular that $\mathcal C^{sp}$ is not empty, and contains algorithms with the static proportion sampling rule for all $w \in \triangle_K$.
Let $H_{\mathcal C^{sp}}$ be the least expected stopping time (normalized by $\log(1/\delta)$) for algorithms in $\mathcal C^{sp}$:
$
H_{\mathcal C^{sp}}(\mu) = \inf_{\mathcal A \in \mathcal C^{sp}} \liminf_{\delta \to 0} \frac{\mathbb{E}_{\mu, \mathcal A}[\tau]}{\log(1/\delta)}
\: .
$
Since $\mathcal C^{sp} \subseteq \mathcal C^\delta$, we have $H_{\mathcal C^\delta} \le H_{\mathcal C^{sp}}$. A remarkable property of fixed confidence identification is that these two functions are in fact equal. For each $\mu \in \mathcal D$, there exists oracle static proportions $w^\star(\mu) \in \triangle_K$ and a static proportion algorithm $\mathcal A_{w^\star(\mu)}^{sp}$ parametrized by $w^\star(\mu)$ such that $\liminf_{\delta \to 0} \frac{\mathbb{E}_{\mu, \mathcal A_{w^\star(\mu)}}[\tau]}{\log(1/\delta)} = H_{\mathcal C^\delta}(\mu)$.
The existence of optimal static proportions is used in the design of TnS: the sampling rule ensures that the sampling proportions converge to $w^\star(\mu)$.
To summarize, the class of $\delta$-correct algorithms in fixed confidence identification satisfies the following properties:
\begin{itemize}[noitemsep]
  \item [(C)] It has a complexity $H_{\mathcal C^\delta}$ which defines a lower bound for all $\mu \in \mathcal D$ and all $\mathcal A \in \mathcal C^\delta$ and there is an algorithm in $\mathcal C^\delta$ that attains it for all $\mu \in \mathcal D$.
  \item [(SP)] The complexity $H_{\mathcal C^\delta}$ is equal to $H_{\mathcal C^{sp}}$, which characterizes the difficulty of each $\mu \in \mathcal D$ for the best static proportions algorithm in hindsight.
\end{itemize}

The description above gives a good picture of asymptotic fixed confidence, in the regime $\delta \to 0$. It is now the object of a large literature, which also deals with structured BAI, other identification tasks, and/or give algorithms that have advantages over TnS.
Fixed confidence BAI with $\delta$ not close to zero and small gaps is also an active field of study, which is less well understood \citep{simchowitz2017simulator,katz2020true}.

\subsection{Fixed Budget Bandit Identification}
\label{sub:fixed_budget_bandit_identification}

An algorithm family $\mathcal A$ is a sequence $(\mathcal A_T)_{T \ge 1}$ of algorithms, one for each possible value of the horizon.
That definition allows us to describe the behavior of fixed budget algorithms in the limit $T \to +\infty$.
This is similar to fixed confidence, where we describe the limit as $\delta\to 0$ of $\mathbb{E}_\mu[\tau]/\log(1/\delta)$: we compute that limit for a family of algorithms, one for each $\delta$.
A good fixed budget algorithm family minimizes the probability of error $p_{\mu, T}$ for all $\mu \in \mathcal D$.
That probability is exponentially small in $T$ for any algorithm that pulls all arms linearly and recommends the empirical correct answer. We hence look at the rate at which it decreases, and define
$h_{\mu, T}(\mathcal A) = T / \log(1/p_{\mu, T}(\mathcal A))$ .
Written differently, the error probability of $\mathcal A$ on $\mu \in \mathcal D$ is $p_{\mu, T}(\mathcal A) = \exp(-T/h_{\mu, T}(\mathcal A))$.

\paragraph{Oracle difficulty of an algorithm class}
\label{par:oracle_difficulty_of_an_algorithm_class}

We call a set of algorithm families an \emph{algorithm class}.
We want to quantify the performance of the best algorithm family in $\mathcal C$ at $\mu \in \mathcal D$.
An algorithm family $\mathcal A$ is asymptotically ``good'' if eventually as $T \to + \infty$, $h_{\mu, T}(\mathcal A)$ becomes small.
We are thus interested $\limsup_{T \to +\infty} h_{\mu, T}(\mathcal A)$.
For an algorithm class, we want to quantify that limsup for the best algorithm in the class, hence we define the oracle difficulty as 
\begin{align*}
H_{\mathcal C}(\mu)
:= \inf_{\mathcal A \in \mathcal C} \limsup_{T \to +\infty} h_{\mu, T}(\mathcal A)
= \inf_{\mathcal A \in \mathcal C} \limsup_{T \to +\infty} T / \log(1/p_{\mu, T}(\mathcal A))
\: .
\end{align*}
We call $H_{\mathcal C}(\mu)$ an \emph{oracle} difficulty because it reflects how difficult the problem $\mu$ is for the algorithm family in the class which is best at $\mu$.
By definition, for all $\mathcal A \in \mathcal C$ and for all $\varepsilon > 0$, there exists infinitely many times $T \ge T_\varepsilon$ such that $p_{\mu, T}(\mathcal A) \ge \exp\left( - T/ (H_{\mathcal C}(\mu) - \varepsilon) \right)$ .
Thus $H_{\mathcal C}$ represents a lower bound on the probability of error of any algorithm family in the class.

\paragraph{Complexity}
\label{par:complexity}

By analogy with fixed confidence identification, we say that an algorithm class $\mathcal C$ admits a complexity if there exists $\mathcal A^\star_{\mathcal C} \in \mathcal C$ such that for all $\mu \in \mathcal D$,
$
\limsup_{T \to +\infty} h_{\mu, T}(\mathcal A^\star_{\mathcal C})
\le H_{\mathcal C}(\mu)
\: .
$
We then have equality and furthermore $H_{\mathcal C} = H_{\{\mathcal A^\star_{\mathcal C}\}}$. We thus say that the class has an asymptotic complexity if a single algorithm matches the lower bound everywhere on $\mathcal D$.
Some classes admit complexities, for example any singleton class, while we will see that others do not.

\paragraph{Difficulty ratio}
\label{par:difficulty_ratio}

In order to establish whether a class admits a complexity, we will need to compare the rate of algorithm families with the difficulty of the class.
Suppose more generally that we are given a function $H: \mathcal D \to \mathbb{R}^+$ which represents a difficulty \emph{a priori} of each $\mu \in \mathcal D$, and that we want to compare $h_{\mu, T}(\mathcal A)$ to $H(\mu)$ in order to assess how good $\mathcal A$ is when compared to the baseline $H$. That function $H$ which will usually be the oracle difficulty of an algorithm class, but not necessarily.
Most of the literature on sub-Gaussian BAI defines $H$ as the sum of the inverse squares of the gaps, and compares algorithms to that baseline.
We define the \emph{difficulty ratio} of an algorithm family $\mathcal A$ to $H$ at a problem $\mu \in \mathcal D$ at time $T$ as
\begin{align*}
R_{H, T}(\mathcal A,\mu) = \frac{h_{\mu, T}(\mathcal A)}{H(\mu)} = \frac{T}{H(\mu) \log(1/p_{\mu, T}(\mathcal A))}
\: .
\end{align*}
That ratio is larger than 1 if $\mathcal A_T$ has error probability larger than the value $\exp(-T/H(\mu))$ prescribed by the difficulty $H$.
If we consider two classes $\mathcal C \subseteq \mathcal C'$, then $H_{\mathcal C} \ge H_{\mathcal C'}$ and $R_{H_{\mathcal C}, T}(\mathcal A,\mu) \le R_{H_{\mathcal C'}, T}(\mathcal A,\mu)$.
We introduce the notation $R_{H,\infty}(\mathcal A,\mu) = \limsup_{T \to \infty} R_{\mu, T}(\mathcal A, \mu)$.
We call the value $\sup_{\mu \in \mathcal D} R_{H_{\mathcal C}, \infty}(\mathcal A, \mu)$ the \emph{maximal difficulty ratio} of $\mathcal A$.

An algorithm class $\mathcal C$ admits an asymptotic complexity iff there exists $\mathcal A^\star_{\mathcal C} \in \mathcal C$ such that $\sup_{\mu \in \mathcal D} R_{H_{\mathcal C}, \infty}(\mathcal A^\star_{\mathcal C}, \mu) \le 1$. If on the contrary that quantity is strictly greater than 1 for all $\mathcal A \in \mathcal C$, then any algorithm in the class has a sub-optimal rate compared to the oracle at some point of $\mathcal D$.

\subsection{Contributions and structure of the paper}
\label{sub:contributions}

We are inspired by the open problem presented at COLT 2022 by \cite{qin2022open}. With our terminology, they ask whether there exists a sufficiently large algorithm class that admits a complexity in fixed budget best arm identification. We draw a parallel with the fixed confidence setting and also ask whether that complexity necessarily equates the oracle difficulty of static proportions.
\begin{itemize}[noitemsep]
  \item We formalized in the introduction the notion of complexity of fixed budget identification and we give tools for the study of that complexity. In particular, we reduce the question of its existence to the derivation of a bound on the difficulty ratio.
  \item In Section~\ref{sec:main_tool}, we present generic lower bounds on the difficulty ratio.
  \item In Section~\ref{sec:the_range_of_the_difficulty_ratio}, we use these tools to study the range of the smallest possible maximal difficulty ratio for any algorithm when compared to static proportions algorithms. We show that this ratio is at least 1 for most tasks, and is at most $K$. The lower bound of 1 indicates that static proportions oracles indeed define lower bounds on the error probability of any algorithm: if a class $\mathcal C$ contains static proportions algorithms and has a complexity, then that complexity is the oracle difficulty of static proportions. The upper bound of $K$ is attained: in the positivity task, uniform sampling is optimal and has a maximal difficulty ratio equal to $K$. 
  \item In Section~\ref{sec:no_complexity_in_best_arm_identification}, we show that for any algorithm class that contains the static proportions algorithms, BAI has no complexity for $K$ large enough. We show that for the same classes, Bernoulli BAI has no complexity for $K=2$.
\end{itemize}

\section{Algorithmic classes}
\label{sec:algorithmic_classes}

We introduce several algorithm classes for which we will ask whether a complexity exists. We denote by $\mathcal C_\infty$ the class of all algorithm families.

\paragraph{Static proportions}
\label{par:static_proportions}
Static proportions algorithms pull all arms according to a pre-defined allocation vector in the simplex, then return the empirical correct answer. That is, $\hat{i}_T = i^\star(\hat{\mu}_T)$.
Let $\triangle_K^0 = \{\omega \in \triangle_K \mid \forall k \in [K], \ \omega_k > 0\}$.
A static proportions algorithm parametrized by $\omega \in \triangle_K^0$ is any sampling rule which satisfies
$\vert N_{T,k} - T \omega_k \vert \le K$ for all $k \in [K]$. Such a sampling rule exists: see the tracking procedure of \cite{garivier2016optimal}, and the bound on the difference $\vert N_{T,k} - T \omega_k \vert$ for that procedure derived by \cite{degenne2020structure}.

Let $\alt(\mu) = \{\lambda \in \mathcal D \mid i^\star(\lambda) \ne i^\star(\mu)\}$ be the set of \emph{alternatives} to $\mu \in \mathcal D$. For $\lambda_k, \mu_k$ two means of distributions in an exponential family, we denote by $\KL(\lambda_k, \mu_k)$ the Kullback-Leibler divergence between the two corresponding distributions. We give now a bound on the probability of error of static proportions algorithms, which is adapted from \citep{glynn2004large}.

\begin{theorem}\label{thm:oracle_difficulty_sp}
Let $\mathcal A_\omega^{sp}$ be a static proportions algorithm parametrized by $\omega \in \triangle_K^0$. For all $\mu \in \mathcal D$,
\begin{align*}
\lim_{T \to +\infty}h_{\mu, T}(\mathcal A^{sp}_\omega) = \Big(\inf_{\lambda \in \alt(\mu)} \sum_{k \in [K]} \omega_k \KL(\lambda_k, \mu_k) \Big)^{-1}
\: .
\end{align*}
\end{theorem}
As a consequence, the oracle difficulty of the class $\mathcal C^{sp}$ of static proportions algorithms is
\begin{align*}
H_{\mathcal C^{sp}}(\mu)
= \inf_{\omega \in \triangle_K^0} \lim_{T \to +\infty}h_{\mu, T}(\mathcal A^{sp}_\omega)
= \Big( \max_{\omega \in \triangle_K^0} \inf_{\lambda \in \alt(\mu)} \sum_{k \in [K]} \omega_k \KL(\lambda_k, \mu_k) \Big)^{-1}
\: .
\end{align*}

Let's illustrate that difficulty on the BAI task with Gaussians distributions with variance 1. For $k \in [K]$, let $\Delta_k = \mu_{i^\star(\mu)} - \mu_k$. It was shown by \cite{garivier2016optimal} that for all $\mu \in \mathcal D$, $H_{\mathcal C^{sp}}$ satisfies the inequalities $H_\Delta(\mu) \le H_{C^{sp}}(\mu) \le 2 H_\Delta(\mu)$, where $H_{\Delta}(\mu) = \frac{2}{\min_{k : \Delta_k>0} \Delta_k^2} + \sum_{k:\Delta_k>0} \frac{2}{\Delta_k^2}$.

\paragraph{Consistent and exponentially consistent}
\label{par:consistent}

An algorithm family is said to be \emph{consistent} \citep{kaufmann2016complexity} if for all $\mu \in \mathcal D$, $\lim_{T\to + \infty} p_{\mu, T} = 0$. We denote that class by $\mathcal C_c$.
It is said to be \emph{exponentially consistent} \citep{barrier2022best} if for all $\mu \in \mathcal D$, $\limsup_{T \to +\infty} h_{\mu, T}(\mathcal A) < +\infty$. We denote that class $\mathcal C_{ec}$.
Consistent algorithms are the largest class of algorithm families which are ``good everywhere'', in the sense that they eventually get the right answer with high probability, no matter which problem $\mu \in \mathcal D$ they face.
Any exponentially consistent algorithm is consistent: $\mathcal C_{ec} \subseteq \mathcal C_c$.
Static proportions algorithms are exponentially consistent: $\mathcal C^{sp} \subseteq \mathcal C_{ec}$. Indeed for any $\omega \in \triangle_K^0$, under Assumption~\ref{asm:answer_nhds} the formula for $\lim_{T \to + \infty} h_{\mu, T}(\mathcal A_\omega^{sp})$ of Theorem~\ref{thm:oracle_difficulty_sp} gives a finite value. This proves that $\mathcal A_\omega^{sp} \in \mathcal C_{ec}$ for all $\omega \in \triangle_K^0$.
We restricted the static proportions to $\triangle_K^0$ instead of $\triangle_K$ to ensure that the algorithms are exponentially consistent.

\paragraph{Bounded difficulty}
\label{par:bounded_difficulty}

The approach of most fixed budget papers, which is however often not explicitly stated like this, is to suppose that some function $H: \mathcal D \to \mathbb{R}$ represents a complexity of the fixed budget identification task and to look for algorithms that have error probability close to $\exp(-T/H(\mu))$. Such a function can be for example $H_\Delta(\mu)$ (defined in the static proportions paragraph) for best arm identification. The algorithms Successive Rejects \citep{audibert2010best} or Successive Halving \citep{karnin2013almost} thus achieve error bounds that depend on $H_\Delta$.
\cite{komiyama2022minimax} make that approach explicit: a possibly arbitrary function $H$ is considered and where we are interested in the following class.
\begin{align*}
\mathcal C(H)
&= \{\mathcal A \mid \exists R \in \mathbb{R}, \: \forall \mu \in \mathcal D, \ \limsup_{T \to \infty}h_{\mu, T}(\mathcal A) \le R H(\mu)\}
= \{\mathcal A \mid \sup_{\mu \in \mathcal D} R_{H,\infty}(\mathcal A,\mu) < +\infty\}
\: .
\end{align*}
We don't allow $H$ to be infinite in $\mathcal D$, which means in particular that $\mathcal C(H) \subseteq \mathcal C_{ec}$ for all $H$.
Of course if $H$ is chosen badly that class will be empty.
The goal of \cite{komiyama2022minimax} is then to design algorithms which get the smallest maximal difficulty ratio, given an arbitrary function $H$.
They derive a theoretical algorithm for which the ratio approaches a proxy of the lower bound (but which is computationally intractable), and introduce a second heuristic based on neural networks.

Given an algorithm class $\mathcal C'$, we will consider its oracle difficulty $H_{\mathcal C'}$ and then the class $\mathcal C(H_{\mathcal C'})$ of algorithms with bounded difficulty ratio with respect to $H_{\mathcal C'}$. We denote $\mathcal C(H_{\mathcal C'})$ by $\overline{\mathcal C'}$.
The class $\overline{\mathcal C'}$ might not contain $\mathcal C'$. If $\mathcal C' \subseteq \mathcal C''$, then from the definition we get $\overline{\mathcal C''} \subseteq \overline{\mathcal C'}$.
The class of static proportions satisfies $\mathcal C^{sp} \subseteq \overline{\mathcal C^{sp}}$.
The proof is a simple study of the ratio between $\lim_{T \to + \infty} h_{\mu, T}(\mathcal A_\omega^{sp})$ for different values of $\omega$. See the proof of Theorem~\ref{thm:uniform_upper_bound} in Section~\ref{sec:the_range_of_the_difficulty_ratio}.

\paragraph{Within a constant of the uniform allocation}

The uniform static proportions algorithm $\mathcal A_u := \mathcal A_{(1/K, \ldots, 1/K)}^{sp} \in \mathcal C^{sp}$, that allocates an equal number of samples to every arm, is a natural baseline to which we can compare algorithms.
We can for example look for algorithms that have a difficulty ratio to the complexity of the uniform allocation which is uniformly bounded on $\mathcal D$. This is the class $\overline{\{\mathcal A_u\}} = \mathcal C(H_{\{\mathcal A_u\}})$.
Since $\mathcal C^{sp} \subseteq \overline{\mathcal C^{sp}}$ and $\{\mathcal A_u\} \subseteq \mathcal C^{sp}$, that class satisfies $\mathcal C^{sp} \subseteq \overline{\mathcal C^{sp}} \subseteq \overline{\{\mathcal A_u\}}$ .

\paragraph{Summary}
\label{par:summary}

Consistent, exponentially consistent algorithms and the class of algorithm families within a constant of the uniform allocation all contain the static proportions algorithms $\mathcal C^{sp}$ : $\mathcal C^{sp} \subseteq \mathcal C_{ec} \subseteq \mathcal C_c$ and $\mathcal C^{sp} \subseteq \overline{\{\mathcal A_u\}}$.
If we get a lower bound on $R_{H_{\mathcal C^{sp}}, T}(\mathcal A, \mu)$ for an algorithm family $\mathcal A$, then it is also a lower bound for the ratio to the difficulty of any of the classes $\mathcal C_{c}$, $\mathcal C_{ec}$, $\overline{\{\mathcal A_u\}}$.

\section{Lower bounds on the difficulty ratio}
\label{sec:main_tool}

Most of the bounds on the difficulty ratio we derive are consequences of the following theorem.
\begin{theorem}\label{thm:R_LB}
Let $H : \mathcal D \to \mathbb{R}^+$ be an arbitrary difficulty function.
Let $\mu, \lambda \in \mathcal D$ be such that $i^*(\lambda) \ne i^*(\mu)$ and $H(\lambda) \le \sqrt{T}$. Then for any algorithm $\mathcal A$,
\begin{align*}
R_{H,T}(\mathcal A, \lambda)^{-1} (1 - p_{\mu,T}(\mathcal A)) - \frac{\log 2}{\sqrt{T}}
\le H(\lambda)\sum_{k=1}^K \mathbb{E}_\mu \left[ \frac{N_{T,k}}{T} \right] \KL(\mu_k, \lambda_k)
\: .
\end{align*}
\end{theorem}

The proof of this inequality follows the standard bandit lower bound argument, appealing to the data processing inequality for the KL divergence, which can be found for example in \citep{garivier2019explore}.
The proof is in Appendix~\ref{sec:proofs_relative_to_section_sec:main_tool}.
The only mildly original step is to put $H(\lambda)$ on the right of the inequality instead of writing a lower bound on $p_{\lambda, T}(\mathcal A)$ (which would give a bound akin to Lemma 6 of \citep{barrier2022best} when taking the limit as $T \to + \infty$).

\begin{theorem}\label{thm:asm_R_LB_limsup}
For any consistent algorithm family $\mathcal A$, for all $\mu \in \mathcal D$ and all sets $D(\mu) \subseteq \alt(\mu)$,
\begin{align*}
(\sup_{\lambda \in D(\mu)} R_{H,\infty}(\mathcal A, \lambda))^{-1}
&\le \max_{\omega \in \triangle_K} \inf_{\lambda \in D(\mu)} H(\lambda)\sum_{k=1}^K \omega_k \KL(\mu_k, \lambda_k)
\: .
\\
\text{Furthermore}, \quad(\sup_{\lambda \in \mathcal D} R_{H,\infty}(\mathcal A, \lambda))^{-1}
&\le \inf_{\mu \in \mathcal D} \max_{\omega \in \triangle_K} \inf_{\lambda \in \alt(\mu)} H(\lambda)\sum_{k=1}^K \omega_k \KL(\mu_k, \lambda_k)
\: .
\end{align*}
\end{theorem}
\begin{proof}
Let $\mu \in \mathcal D$. Since $\triangle_K$ is compact, the sequence $(\mathbb{E}_{\mu, \mathcal A}[N_{T}/T])_{T \in \mathbb{N}}$ has a subsequence indexed by some $(T_n)_{n \in \mathbb{N}}$ which converges to a vector $\omega_{\mu} \in \triangle_K$. Let $\lambda \in \alt(\mu)$. Theorem~\ref{thm:R_LB} gives, for $n$ large enough,
\begin{align*}
R_{H,T_n}(\mathcal A, \lambda)^{-1} (1 - p_{\mu, T_n}(\mathcal A)) - \frac{\log 2}{\sqrt{T_n}}
\le H(\lambda)\sum_{k=1}^K \mathbb{E}_\mu \left[ \frac{N_{T_n,k}}{T_n} \right] \KL(\mu_k, \lambda_k)
\: .
\end{align*}
Since $\mathcal A$ is consistent, $1 - p_{\mu, T_n}(\mathcal A) \to 1$. Taking a limit as $n \to +\infty$, we have
\begin{align*}
\liminf_{n \to + \infty} R_{H,T_n}(\mathcal A, \lambda)^{-1}
\le H(\lambda)\sum_{k=1}^K \omega_{\mu, k} \KL(\mu_k, \lambda_k)
\: .
\end{align*}
That bound on the liminf of a subsequence gives a bound on the liminf of the whole sequence. We finally take an infimum over $\lambda \in D(\mu)$ on both sides of the inequality, and replace $\omega_{\mu}$ by a maximum over the simplex. We proved the first statement. The second inequality is obtained by choosing $D(\mu) = \alt(\mu)$ and taking an infimum over $\mu \in \mathcal D$.
\end{proof}

The second inequality of Theorem~\ref{thm:asm_R_LB_limsup} recovers Theorem 1 of \citep{komiyama2022minimax}, at least under our assumptions (their hypotheses on $\mathcal D$ are not as strict as ours). They prove it differently: they introduce typical concentration events, reduce the study to those events and use a change of measure.
Their proof does not give an explicit non-asymptotic version of the bound, unlike Theorem~\ref{thm:R_LB}. 
In contrast, our short proof is a direct application of the data processing inequality for the KL divergence.

Instead of an inequality on the supremum of the limsup of $R_{H, T}(\mathcal A, \mu)$ as in Theorem~\ref{thm:asm_R_LB_limsup}, we can also get a bound on the liminf of the supremum of $R_{H, T}(\mathcal A, \mu)$ over sets with bounded $H$. See Theorem~\ref{thm:asm_R_LB} in Appendix~\ref{sec:proofs_relative_to_section_sec:main_tool}.
We will use Theorem~\ref{thm:asm_R_LB_limsup} in order to describe the asymptotic difficulty of fixed budget identification. We could derive bounds for a fixed $T$ by using Theorem~\ref{thm:R_LB} instead, at the cost of second order terms and restrictions of the alternative to problems with $H$ bounded by $\sqrt{T}$, that is to problems which are not too hard at time $T$. 
\begin{corollary}\label{cor:corner_lb}
Let $\mu, \lambda^{(1)}, \ldots, \lambda^{(K)}$ be such that for all $j \in [K]$, $i^\star(\lambda^{(j)}) \ne i^\star(\mu)$, $H(\lambda^{(j)})>0$, and each $\lambda^{(j)}$ differ from $\mu$ only along coordinate $j$. Then for all algorithms $\mathcal A$ such that $\lim_{T \to +\infty} p_{\mu, T}(\mathcal A) = 0$ ,
\begin{align*}
\sup_{j \in [K]} R_{H,\infty}(\mathcal A,\lambda^{(j)})
\ge \sum_{j=1}^K \frac{1}{H(\lambda^{(j)}) \KL(\mu_j, \lambda_j^{(j)})}
\: .
\end{align*}
\end{corollary}

\begin{proof}
We apply the first inequality of Theorem~\ref{thm:asm_R_LB_limsup} with $D(\mu) = \{\lambda^{(1)}, \ldots, \lambda^{(K)}\}$.
\begin{align*}
(\sup_{j \in [K]} R_{H,\infty}(\mathcal A, \lambda^{(j)}))^{-1}
&\le \max_{\omega \in \triangle_K} \inf_{j \in [K]} H(\lambda^{(j)})\sum_{k=1}^K \omega_k \KL(\mu_k, \lambda_k^{(j)})
\\
&= \max_{\omega \in \triangle_K} \inf_{j \in [K]} H(\lambda^{(j)}) \omega_j \KL(\mu_j, \lambda_j^{(j)})
\: .
\end{align*}
The optimal $\omega$ equalizes $H(\lambda^{(j)}) \omega_j \KL(\mu_j, \lambda_j^{(j)})$ for all $j$, which gives the result.
\end{proof}

The sum on the right hand side of Corollary~\ref{cor:corner_lb} is very close to the quantity $h^*$ defined in \citep{carpentier2016tight} in the setting of Bernoulli bandits with $H$ the sum of inverse squared gaps. This is due to the similar construction of a set of points in the alternative that each differ from a given $\mu \in \mathcal D$ in one coordinate only. That construction was reused by \cite{ariu2021policy} to get a bound on a quantity called expected policy regret and by \cite{yangminimax} to prove a lower bound for fixed budget BAI in linear bandits.

The main advantage of Corollary~\ref{cor:corner_lb} is that it is simpler to use than Theorem~\ref{thm:asm_R_LB_limsup}, but it can lead to worse bounds. For example in BAI in two-arms Gaussian bandits with known variance 1, with $H = H_{\mathcal C^{sp}}$ Theorem~\ref{thm:asm_R_LB_limsup} gives $\sup_{\lambda\in \mathcal D} R_{H, \infty}(\mathcal A, \lambda) \ge 1$ while the best bound that can be achieved with Corollary~\ref{cor:corner_lb} is 1/2.
That task is very simple, as remarked by \cite{kaufmann2016complexity}: the oracle fixed proportions are independent of the means (both arms are played equally), which means that the algorithm that plays those proportions has $\sup_{\lambda\in \mathcal D} R_{H, \infty}(\mathcal A, \lambda) \le 1$.
Theorem~\ref{thm:asm_R_LB_limsup} shows that this is tight and that no adaptive algorithm can beat it everywhere.
We could not arrive to that conclusion with the weaker Corollary~\ref{cor:corner_lb} since it only proves a $1/2$ lower bound.

\section{The range of the difficulty ratio}
\label{sec:the_range_of_the_difficulty_ratio}

In asymptotic fixed confidence, the complexity of $\delta$-correct algorithms is given by the oracle difficulty of static proportions. There is an optimal sampling allocation at each $\mu \in \mathcal D$, and the best any adaptive algorithm can do is match the performance of that allocation.
The fixed confidence analogue of the difficulty ratio would be greater than or equal to $1$ for any $\delta$-correct algorithm, and exactly $1$ for TnS.
We hence focus on the ratio of fixed budget algorithm families to the oracle difficulty of the class of static proportions algorithms, which is given by $H_{\mathcal C^{sp}}(\mu) = \left( \max_{\omega \in \triangle_k} \inf_{\lambda \in \alt(\mu)} \sum_{k=1}^K \omega_k \KL(\lambda_k, \mu_k) \right)^{-1}$.
In a general fixed budget identification task described by $(\mathcal D, \mathcal I, i^\star)$, two related questions remain open:
\begin{itemize}[noitemsep]
  \item Do fixed proportions indeed always define oracle algorithms, or could there exist an adaptive algorithm with a better rate everywhere? In technical terms, can we have the inequality $\inf_{\mathcal A \in \mathcal C_\infty} \sup_{\lambda\in \mathcal D} R_{H_{\mathcal C^{sp}},\infty}(\mathcal A, \lambda)  < 1$ ? Recall that $\mathcal C_\infty$ is the class of all algorithm families. \cite{ouhamma2021online} exhibit a setting close to fixed budget identification in which an adaptive algorithm can indeed beat any static proportions algorithm. However, their objective does not fit into our fixed budget identification framework and their example uses families of distributions in which the KL can be infinite.
  \item For Bernoulli BAI, a lower bound of \citep{carpentier2016tight} and the upper bound on the Successive Rejects algorithm of \citep{audibert2010best} together show that for $H_1$ the sum of inverse squared gaps, the value $\inf_{\mathcal A \in \mathcal C_\infty} \sup_{\lambda\in \mathcal D} R_{H_1,\infty}(\mathcal A, \lambda)$ is of order $\log K$, strictly greater than 1 for $K$ large enough. Do we have the same bound for $\mathcal H_{\mathcal C^{sp}}$ and are there problems on which the difficulty ratio can be much larger than $\log K$?
\end{itemize}

We study the possible values for the smallest maximal difficulty ratio over all algorithm: we prove upper and lower bounds on $\inf_{\mathcal A \in \mathcal C_\infty} \sup_{\lambda\in \mathcal D} R_{H_{\mathcal C^{sp}},\infty}(\mathcal A, \lambda) $ when we vary the task $(\mathcal D, \mathcal I, i^\star)$.

\subsection{Upper bound}
\label{sub:upper_bound}

We first prove that $\inf_{\mathcal A \in \mathcal C_\infty} \sup_{\lambda \in \mathcal D}  R_{H_{\mathcal C^{sp}}, \infty}(\mathcal A, \lambda) \le K$ on any task $(\mathcal D, \mathcal I, i^\star)$ by showing that uniform sampling can be worse than the oracle static proportions by a factor of at most $K$. We then exhibit a task on which there is equality.

\begin{theorem}\label{thm:uniform_upper_bound}
For all $\omega \in \triangle_K^0$, the static proportions algorithm $\mathcal A_\omega^{sp}$ belongs to $\overline{\mathcal C^{sp}}$ and satisfies $\sup_{\lambda\in \mathcal D} R_{H_{\mathcal C^{sp}}, \infty}(\mathcal A_\omega^{sp}, \lambda) \le (\min_{j \in [K]} \omega_j)^{-1}$. In particular, for $A_u \in \mathcal C^{sp}$ the uniform sampling algorithm (static proportions with proportion $1/K$ for all arms),
$
\sup_{\lambda\in \mathcal D} R_{H_{\mathcal C^{sp}}, \infty}(\mathcal A_u, \lambda) \le K \: .
$
\end{theorem}
\begin{proof}
Let $\omega^\star(\mu) \in \triangle_K$ be the oracle static proportions at $\mu$ and let $\omega \in \triangle_K^0$. Then for all $k$, $\omega_k \ge \omega^\star_k(\mu) \min_j \omega_j$ and, using Theorem~\ref{thm:oracle_difficulty_sp},
\begin{align*}
\limsup_{T \to +\infty} h_{\mu, T}(\mathcal A_\omega^{sp})
&\le \frac{1}{\min_{j \in [K]} \omega_j}\left( \inf_{\lambda \in \alt(\mu)} \sum_{k=1}^K \omega_k^\star(\mu) \KL(\lambda_k, \mu_k) \right)^{-1}
= \frac{1}{\min_{j \in [K]} \omega_j} H_{\mathcal C^{sp}}(\mu)
\: .
\end{align*}
We proved that $\limsup_{T \to +\infty} R_{H_{\mathcal C^{sp}}, T}(\mathcal A_\omega^{sp}, \mu) \le (\min_{j \in [K]} \omega_j)^{-1}$ for all $\mu \in \mathcal D$.
\end{proof}

Of course there are tasks for which uniform sampling is not the best algorithm: for Gaussian BAI the Successive-Rejects algorithm \citep{audibert2010best} has a ratio of order $\log K$ (see also \citep{barrier2022best}). However, in some identification tasks $K$ is the best achievable ratio.

\begin{theorem}\label{thm:positivity_LB}
On the Positivity problem, where we check whether there is an arm with mean lower than a threshold $\theta$, $\inf_{\mathcal A \in \mathcal C_\infty} \sup_{\lambda \in \mathcal D}  R_{H_{\mathcal C^{sp}}, \infty}(\mathcal A, \lambda) = K$.
\end{theorem}

That theorem proves that on the positivity problem, if a class contains the static proportions algorithms then it does not have a complexity. Furthermore, the uniform sampling algorithm is optimal for the criterion $\sup_{\lambda \in \mathcal D} R_{H_{\mathcal C^{sp}}, \infty}(\mathcal A, \lambda)$.

\begin{proof}
Let $\mathcal A$ be any algorithm family.
We use Corollary~\ref{cor:corner_lb} for $\mu$ a tuple of $K$ times the same distribution with mean $m > \theta$.
Either $R_{H_{\mathcal C^{sp}}, \infty}(\mathcal A, \mu) = + \infty$ and the lower bound is obvious or we can apply the corollary.
For $j \in [K]$, we define $\lambda^{(j)}$ identical to $\mu$ except for $\lambda^{(j)}_j = \ell < \theta$. Then
$
\max_{j \in [K]} R_{H_{\mathcal C^{sp}}, \infty}(\mathcal A, \lambda^{(j)})
\ge \sum_{j = 1}^K (H_{\mathcal C^{sp}}(\lambda^{(j)}) \KL(m, \ell))^{-1}$ .
Now for all $j$, a simple computation gives $H_{\mathcal C^{sp}}(\lambda^{(j)}) = (\KL(\theta, \ell))^{-1}$, such that the lower bound is $K \KL(\theta, \ell) / \KL(m, \ell)$.
When $\ell$ tends to the lower bound of the means in the exponential family, the KL ratio tends to 1.
\end{proof}

The proof of Theorem~\ref{thm:positivity_LB} exhibits $K$ problems, each with a different arm with mean below the threshold, and the oracle algorithm for each samples only that arm. The lower bound shows that detecting which arm is below the threshold is harder than the identification task and that no matter the algorithm, it is as bad as uniform sampling on one of the problems (but we don't know which).

We established that the highest possible value for identification tasks $(\mathcal D, \mathcal I, i^\star)$ of the quantity $\inf_{\mathcal A \in \mathcal C_\infty} \sup_{\lambda \in \mathcal D} R_{H_{\mathcal C^{sp}}, \infty}(\mathcal A, \lambda)$ is $K$, and that this value is attained for the Positivity problem.

\subsection{Lower bounds}
\label{sub:lower_bounds}

We turn our attention to lower bounds. A natural conjecture is the following: \emph{for all fixed budget tasks and all algorithm families,} $\sup_{\lambda \in \mathcal D} R_{H_{\mathcal C^{sp}}, \infty}(\mathcal A, \lambda) \ge 1$.
If true, then no adaptive algorithm that can do everywhere better than the static proportions oracle.
It could still have lower error probability on one problem $\mu \in \mathcal D$, but would have to be worse somewhere else.
First, we prove the conjecture for Gaussian half-space identification (Lemma~\ref{lem:half_space} in Appendix~\ref{sec:proofs_of_results_from_section_sec:the_range_of_the_difficulty_ratio}). In that task, there are two answers and $i^\star$ has a different value on each side of a hyperplane.
We then extend that result to Gaussian distributions with piecewise linear boundaries between the answer sets.

\begin{theorem}\label{thm:flat_boundary}
Suppose that there is an $L^2$ ball $B(\eta, r)$ with center $\eta \in \mathrm{cl}(\mathcal D)$ and radius $r > 0$ such that $i^\star$ takes only two values in $B(\eta, r)$, say $i$ and $j$, and the boundary between $B(\eta, r) \cap \{\mu \mid i^\star(\mu) = i\}$ and $B(\eta, r) \cap \{\mu \mid i^\star(\mu) = j\}$ is the restriction of a hyperplane passing through $\eta$. Then for Gaussian arms (each with a known but possibly different variance), the lowest maximal difficulty ratio is $\inf_{\mathcal A \in \mathcal C_\infty} \sup_{\lambda \in \mathcal D} R_{H_{\mathcal C^{sp}}, \infty}(\mathcal A, \lambda) \ge 1$.
\end{theorem}

The idea of the proof is the following: if we consider $\lambda \in \mathcal D$ close to the center of the ball, then the oracle difficulty $H_{\mathcal C^{sp}}(\lambda)$ of static proportions for our task is the same as for half-space identification.
Then if we choose $\mu$ even closer to the center, we can apply Theorem~\ref{thm:asm_R_LB_limsup} to a set $D(\mu)$ of points for which this equality holds.
Up to border effects that disappear when $\mu$ get closer to the center, we get the same lower bound as for half-space identification.
Full proof in Appendix~\ref{sec:proofs_of_results_from_section_sec:the_range_of_the_difficulty_ratio}.

The hypothesis of that lemma applies to all examples of fixed budget identification we introduced. Indeed BAI, Thresholding bandits and Positivity all have piecewise linear boundaries.
More generally, we could extend Theorem~\ref{thm:flat_boundary} to tasks in which the boundary has bounded curvature at some point: we can zoom in on that point and find problems for which we recover the half-space bound. This remark also illustrates the limitation of Theorem~\ref{thm:flat_boundary}: it is asymptotic in nature.
The proof requires points that are much closer to the center of the ball than the radius.
Either we need a very large ball (BAI when the two best arms have much higher means than other arms) or we need problems very close to the boundary.
It should be possible to extend the theorem to any exponential family by using that locally the KL is quadratic. Again, we would describe the asymptotic behavior of an algorithm family on problems very close to a given boundary point.

The lower bound $\inf_{\mathcal A \in \mathcal C_\infty} \sup_{\lambda \in \mathcal D} R_{H_{\mathcal C^{sp}}, \infty}(\mathcal A, \lambda) \ge 1$ shows that if a class $\mathcal C$ contains $\mathcal C^{sp}$ and admits a complexity, then that complexity has to be $H_{\mathcal C^{sp}}$.

\section{No Complexity in Best Arm Identification}
\label{sec:no_complexity_in_best_arm_identification}

We have investigated the possible values for the difficulty ratio over different identification tasks. We now focus on best arm identification, with $\mathcal I = [K]$ and $i^\star$ the arm with highest mean. We show that for several values of $\mathcal D$, $\inf_{\mathcal A \in \mathcal C_{\infty}}\sup_{\lambda \in \mathcal D} R_{H_{\mathcal C}, \infty}(\mathcal A, \lambda) > 1$ for any class $\mathcal C$ that includes the static proportions algorithms. We conclude that these classes don't admit a complexity.

\subsection{Gaussian best arm identification}
\label{sub:gaussian_best_arm_identification}

\begin{theorem}\label{thm:gaussian_bai_no_complexity}
Consider the BAI task with Gaussian distributions with variance 1.
For any class $\mathcal C$ containing the static proportions algorithms,
$\inf_{\mathcal A \in \mathcal C_\infty} \sup_{\lambda \in \mathcal D} R_{H_{\mathcal C},\infty}(\mathcal A, \lambda)
\ge (3/80)\log(K)$ .
\end{theorem}

This proves that for $K$ large enough, no algorithm class containing the static proportions admits a complexity in Gaussian BAI. It applies to (exponentially) consistent algorithms and to algorithms that have a difficulty ratio to the complexity of the uniform allocation which is uniformly bounded.

\begin{proof}
First, since $\mathcal C^{sp} \subseteq \mathcal C$, for any algorithm $\mathcal A$ and $\mu \in \mathcal D$, $R_{H_{\mathcal C}, T}(\mathcal A, \mu) \ge R_{H_{\mathcal C^{sp}}, T}(\mathcal A, \mu)$. It suffices to give a lower bound for $H_{\mathcal C^{sp}}$.

Let $H_{\Delta}(\mu) = \frac{2}{\min_{k : \Delta_k>0} \Delta_k^2} + \sum_{k:\Delta_k>0} \frac{2}{\Delta_k^2}$. It was shown by \cite{garivier2016optimal} that for all $\mu \in \mathcal D$, this function satisfies the inequalities $H_\Delta(\mu) \le H_{C^{sp}}(\mu) \le 2 H_\Delta(\mu)$ .
Thus $R_{H_{\mathcal C}, T}(\mathcal A, \mu) \ge R_{H_\Delta, T}(\mathcal A, \mu) / 2$. From this point on, we use a construction similar to the one that was used by \cite{carpentier2016tight} to prove a lower bound on the ratio to $H_\Delta$ for Bernoulli bandits.
We define a Gaussian problem $\mu$ by $\mu_1 = 0$ (or any arbitrary value) and $\mu_k = \mu_1 - k \Delta$ for all $k \in \{2, \ldots, K\}$ and some $\Delta > 0$.
We apply Corollary~\ref{cor:corner_lb} to $\mu$ and $\lambda^{(2)}, \ldots, \lambda^{(K)}$ where each $\lambda^{(j)}$ is identical to $\mu$ except that $\lambda^{(j)}_j = \mu_1 + (\mu_1 - \mu_j)$. The details can be found in appendix~\ref{sec:proofs_relative_to_section_sec:no_complexity_in_best_arm_identification}.
\end{proof}

The closest existing result is the lower bound of \citep{carpentier2016tight}.
They don't consider the difficulty of fixed proportions but $H_\Delta$, the sum of inverse squared gaps.
That function was hypothesized to be a complexity for fixed budget at the time. They present a set of Bernoulli problems and show that for all algorithms that return $\hat{i}_T = i^\star(\hat{\mu}_T)$, there is a lower bound on the probability of error on one problem in the set.
Their lower bound can be rewritten as a bound on $\sup_{\lambda \in \mathcal D} R_{H_\Delta,T}(\mathcal A, \lambda)$.
It is not asymptotic in $T$, but we could also obtain a non-asymptotic bound by using Theorem~\ref{thm:R_LB} instead of Theorem~\ref{thm:asm_R_LB_limsup} when deriving Corollary~\ref{cor:corner_lb} at the cost of additional low order terms.
Their result is valid only for algorithms that return the empirical correct answer and does not for example apply to Successive Rejects, while we derive a result for any algorithm.

Since the Kullback-Leibler divergence for other exponential families can be bounded from above and below by a constant times the Gaussian KL if we consider only parameters in a closed bounded interval, we can extend Theorem~\ref{thm:gaussian_bai_no_complexity} beyond Gaussians. We obtain that there exists a constant $c$ such that
$\inf_{\mathcal A \in \mathcal C_\infty} \sup_{\lambda \in \mathcal D} R_{H_{\mathcal C},\infty}(\mathcal A, \lambda)
\ge c \log(K)$ . Hence for $K$ large enough there is no complexity.

\subsection{Two arms best arm identification with Bernoulli distributions}
\label{sub:two_arms}

In BAI with two arms and Gaussian distributions with known variances (possibly different for each arm), there is a unique static proportions oracle, independent of the means \citep{kaufmann2016complexity}. Thus that same algorithm matches the lower bound on all $\mu \in \mathcal D$ and fixed budget BAI with two Gaussian arms has a complexity. We showed that as $K$ becomes large, this is no longer the case.
In Bernoulli bandits, we show that there is no complexity even for $K = 2$. 
From Theorems~\ref{thm:uniform_upper_bound} and~\ref{thm:flat_boundary}, we know that the infimum of the maximal difficulty ratio belongs to the interval $[1, 2]$, where the upper bound comes from $K=2$. We now prove that it is strictly greater than 1. 
We will apply Corollary~\ref{cor:corner_lb} to well chosen mean vectors.
In order to do so, we first compute explicitly the oracle difficulty of static proportions algorithms.

\begin{lemma}\label{cor:bernoulli_bai_H}
In a two arms BAI problem with Bernoulli distributions,
\begin{align*}
(H_{\mathcal C^{sp}}(\mu))^{-1}
&= \KL\Big(\frac{\log\frac{1 - \mu_2}{1 - \mu_1}}{\log \frac{\mu_1(1 - \mu_2)}{(1 - \mu_1)\mu_2}}, \mu_1\Big)
= \KL\Big(\frac{\log\frac{1 - \mu_2}{1 - \mu_1}}{\log \frac{\mu_1(1 - \mu_2)}{(1 - \mu_1)\mu_2}}, \mu_2\Big)
\: .
\end{align*}
\end{lemma}

\begin{theorem}\label{thm:bernoulli_bai_no_complexity}
In BAI for Bernoulli bandits with two arms, for any class $\mathcal C$ containing the static proportions algorithms,
$
\inf_{\mathcal A \in \mathcal C_\infty} \sup_{\lambda \in \mathcal D} R_{H_{\mathcal C}, \infty}(\mathcal A, \lambda)
> 1
$ .
\end{theorem}

The lemma is a special case of a more general result which applies to all exponential families: Lemma~\ref{lem:exp_fam_bai} in Appendix~\ref{sec:proofs_relative_to_section_sec:no_complexity_in_best_arm_identification}.
The proof is an explicit computation.
We now apply Corollary~\ref{cor:corner_lb} to $\mu = (x(1+x), x)$ for some $x\in (0,1/2)$, $\lambda^{(1)} = (x/2, x)$ and $\lambda^{(2)} = (x(1+x), 1/2)$. This gives an explicit lower bound, function of $x$.
The limit of that bound at 0 is approximately $1.22$, which means that there exists $x$ small enough for which it is greater than 1.
Theorem~\ref{thm:bernoulli_bai_no_complexity} is proved (see Appendix~\ref{sec:proofs_relative_to_section_sec:no_complexity_in_best_arm_identification} for details).
Values $x$ for which we get a lower bound greater than 1 are very small, $10^{-9}$ and lower.
We used Corollary~\ref{cor:corner_lb} and not Theorem~\ref{thm:asm_R_LB_limsup} because it allows a closed form computation of the bound, but by doing so we may have lost constants. It is possible that we could show a lower bound greater than 1 for $x$ which is not so close to 0.

\section{Conclusion}
\label{sec:conclusion}

We prove that in most fixed budget identification tasks, if a class containing the static proportions algorithms admits a complexity then it is $H_{\mathcal C^{sp}}$.
However, even in simple tasks like Positivity or BAI with two Bernoulli arms, we showed that there is no such complexity.
For other classes like Thresholding bandits the question is still open. We know that the maximal difficulty ratio of APT \citep{locatelli2016optimal,ouhamma2021online} for Gaussian thresholding bandits is less than an absolute constant, so there is no lower bound that depends on $K$. Another open question is whether there exists a complexity in Gaussian BAI for small $K>2$. We conjecture that there is none.

An important question remains: is there a meaningful class for which there exists a complexity in BAI?
We showed that it would need to exclude some static proportions algorithms. A candidate could be algorithms with difficulty ratio to the uniform allocation less than $n>1$. That class contains $\mathcal C_{1/n}^{sp}$, static proportions with $\min_k \omega_k \ge 1/n$.
We can show $(1 - 1/n) H_{\mathcal C^{sp}}^{-1} \le H_{\mathcal C^{sp}_{1/n}}^{-1} \le H_{\mathcal C^{sp}}^{-1}$, which means that a lower bound of 1 for $H_{\mathcal C^{sp}}$ would give a $(1 - 1/n)$ bound here: an adaptive algorithm could possibly beat all such static allocations everywhere, but only by that constant factor.

If there is no complexity, there can be many ``good'' algorithms. First, we could look for algorithms with smallest maximal difficulty ratio, as pioneered by \cite{komiyama2022minimax}. Successive Rejects is such an algorithm for Gaussian BAI. Then we may want to design methods that are better than the minimax lower bound on some parts of the space (and necessarily worse elsewhere). Can we design an algorithm that sacrifices performance on very easy problems in order to beat the lower bound on more interesting instances?

\acks{The author acknowledges the funding of the French National Research Agency under the project FATE (ANR-22-CE23-0016-01).
This work beneficiated from the support of the French Ministry of Higher Education and Research, of Inria and of the Hauts-de-France region.
The author is part of the Inria Scool team.}

\bibliography{bib}

\appendix

\section{Proofs of results from Section~\ref{sec:algorithmic_classes}}
\label{sec:proofs_relative_to_section_sec:algorithmic_classes}

\paragraph{Proof of Theorem \ref{thm:oracle_difficulty_sp}}

The empirical mean in canonical exponential families satisfies a large deviation principle (LDP).

\begin{lemma}\label{lem:ldp_single_arm}
Let $\mu_k$ be the mean of a distribution in a canonical one-parameter exponential family. Then the empirical mean $\hat{\mu}_{T,k}$ of $T$ samples of that distribution obeys an LDP with rate $T$ and good rate function $x \mapsto \KL(x, \mu_k)$.
\end{lemma}

Let $\mathrm{int} S$ be the interior of a set $S$, and $\mathrm{cl} S$ be its closure.
An application of the Gärtner-Ellis theorem, as done in \cite{glynn2004large}, leads to the following theorem.
\begin{theorem}\label{thm:ldp}
Let $\mathcal A_\omega^{sp}$ be a static proportions algorithm parametrized by $\omega \in \triangle_K^0$. On problem $\mu \in \mathcal D$, the empirical mean vector $\hat{\mu}_T$ obeys a LDP with rate $T$ and good rate function $\lambda \mapsto \sum_{k=1}^K \omega_k \KL(\lambda_k, \mu_k)$. As a consequence, for any set $S \subseteq \mathbb{R}^K$,
\begin{align*}
- \inf_{\lambda \in \mathrm{int} S} \sum_{k=1}^K \omega_k \KL(\lambda_k, \mu_k)
\le \liminf_{T\to +\infty} \frac{1}{T} \log \mathbb{P}_{\mu, \mathcal A_\omega^{sp}}(\hat{\mu}_T \in S) \: ,
\\
\limsup_{T\to +\infty} \frac{1}{T} \log \mathbb{P}_{\mu, \mathcal A_\omega^{sp}}(\hat{\mu}_T \in S) 
\le - \inf_{\lambda \in \mathrm{cl} S} \sum_{k=1}^K \omega_k \KL(\lambda_k, \mu_k)
\: .
\end{align*}
\end{theorem}

By continuity of the Kullback-Leibler divergence in exponential families, for all $\mu \in \mathcal D$ and $\omega \in \triangle_K$ the infimum over the interior and the closure are equal to the infimum over the set.
Thus, the LDP of Theorem~\ref{thm:ldp} gives the equality 
\begin{align*}
\lim_{T \to +\infty}h_{\mu, T}(\mathcal A^{sp}_\omega)
&= \lim_{T \to +\infty}\left(- \frac{1}{T} \log \mathbb{P}_{\mu, \mathcal A_\omega^{sp}}(\hat{\mu}_T \in \alt(\mu)) \right)^{-1}
\\
&= \left(\inf_{\lambda \in \alt(\mu)} \sum_{k=1}^K \omega_k \KL(\lambda_k, \mu_k) \right)^{-1}
\: .
\end{align*}

\section{Proofs of results from Section~\ref{sec:main_tool}}
\label{sec:proofs_relative_to_section_sec:main_tool}

\subsection{Proof of the lower bound Theorem~\ref{thm:R_LB}}
\label{sub:main_theorem}

\begin{proof}[of Theorem~\ref{thm:R_LB}]
The proof of this inequality follows the standard bandit lower bound argument, which can be found for example in \cite{garivier2019explore}.
The Kullback-Leibler divergence between the observations up to $T$ under models $\mu$ and $\lambda$ is $\sum_{k=1}^K \mathbb{E}_\mu[N_{T,k}] \KL(\mu_k, \lambda_k)$.
By the data processing inequality, this Kullback-Leibler divergence is larger than the KL between Bernoulli distributions of means $\mathbb{P}_{\mu, \mathcal A}(E)$ and $\mathbb{P}_{\lambda, \mathcal A}(E)$ for any event $E$. We apply this to $E = \{\hat{i}_T = i^\star(\mu)\}$ to obtain
\begin{align*}
\kl(\mathbb{P}_{\mu, \mathcal A}(\hat{i}_T = i^\star(\mu)), \mathbb{P}_{\lambda, \mathcal A}(\hat{i}_T = i^\star(\mu)))
\le \sum_{k=1}^K \mathbb{E}_\mu[N_{T,k}] \KL(\mu_k, \lambda_k)
\: .
\end{align*}
We use the inequality $\kl(a, b) \ge a \log \frac{1}{b} - \log 2$, then $\mathbb{P}_{\mu, \mathcal A}(\hat{i}_T = i^\star(\mu)) = 1 - p_{\mu, T}(\mathcal A)$, $\mathbb{P}_{\lambda, \mathcal A}(\hat{i}_T = i^\star(\mu)) \le p_{\lambda, T}(\mathcal A)$ (since $i^\star(\lambda) \ne i^\star(\mu)$) to get
\begin{align*}
(1 - p_{\mu, T}(\mathcal A)) \log \frac{1}{p_{\lambda, T}(\mathcal A)} - \log 2
\le \sum_{k=1}^K \mathbb{E}_\mu[N_{T,k}] \KL(\mu_k, \lambda_k)
\: .
\end{align*}
By definition, $p_{\lambda, T}(\mathcal A) = \exp(-T R_{H,T}(\mathcal A, \lambda)^{-1} H(\lambda)^{-1})$,. We get
\begin{align*}
(1 - p_{\mu, T}(\mathcal A)) T R_{H,T}(\mathcal A, \lambda)^{-1} H(\lambda)^{-1} - \log 2
\le \sum_{k=1}^K \mathbb{E}_\mu[N_{T,k}] \KL(\mu_k, \lambda_k)
\: .
\end{align*}
Dividing by $T H(\lambda)^{-1}$ and using $H(\lambda) \le \sqrt{T}$ gives the result.
\end{proof}

\subsection{Additional results}
\label{sub:additional_results}

\begin{theorem}\label{thm:asm_R_LB}
Let $\mu \in \mathcal D$ and let $\mathcal A$ be an algorithm with $\lim_{T \to + \infty} p_{\mu, T}(\mathcal A) = 0$. Let $D(\mu) \subseteq \alt(\mu)$ be a set such that $\sup_{\lambda \in D(\mu)} H(\lambda) < +\infty$. Then
\begin{align*}
(\liminf_{T\to +\infty}\sup_{\lambda \in D(\mu)}R_{H,T}(\mathcal A, \lambda))^{-1}
\le \max_{\omega \in \triangle_K} \inf_{\lambda \in D(\mu)} H(\lambda)\sum_{k=1}^K \omega_k \KL(\mu_k, \lambda_k)
\: .
\end{align*}
\end{theorem}
If $\mathcal A$ is consistent, then it satisfies in particular the condition of the theorem $\lim_{T \to + \infty} p_{\mu, T}(\mathcal A) = 0$.

\begin{proof}
For $T$ large enough, we can apply Theorem~\ref{thm:R_LB} for any $\lambda \in D(\mu)$, hence we can take an infimum over $\lambda \in D(\mu)$ to get
\begin{align*}
(\sup_{\lambda \in D(\mu)}R_{H,T}(\mathcal A, \lambda))^{-1} (1 - p_{\mu, T}(\mathcal A)) - \frac{\log 2}{\sqrt{T}}
&\le \inf_{\lambda \in D(\mu)} H(\lambda)\sum_{k=1}^K \mathbb{E}_\mu[\frac{N_{T,k}}{T}] \KL(\mu_k, \lambda_k)
\\
&\le \max_{\omega \in \triangle_K} \inf_{\lambda \in D(\mu)} H(\lambda)\sum_{k=1}^K \omega_k \KL(\mu_k, \lambda_k)
\: .
\end{align*}
Taking a limit when $T \to +\infty$ and using $\lim_{T \to + \infty} p_{\mu, T}(\mathcal A) = 0$, we get the inequality we want to prove.
\end{proof}

\begin{corollary}
For all $x \in \mathbb{R}$, let $\alt_x(\mu) = \alt(\mu) \cap \{\lambda \in \mathcal D \mid H(\lambda) \le x\}$. For all consistent algorithm families $\mathcal A$,
\begin{align*}
(\liminf_{T\to +\infty}\sup_{\lambda \in \mathcal D}R_{H,T}(\mathcal A, \lambda))^{-1}
\le \liminf_{x \to \infty} \inf_{\mu \in \mathcal D} \max_{\omega \in \triangle_K} \inf_{\lambda \in \alt_x(\mu)} H(\lambda)\sum_{k=1}^K \omega_k \KL(\mu_k, \lambda_k)
\: .
\end{align*}
\end{corollary}

\begin{proof}
Let $\mu \in \mathcal D$ and $x > 0$. We apply Theorem~\ref{thm:asm_R_LB} to $\alt_{x}(\mu)$.
\begin{align*}
(\liminf_{T\to +\infty}\sup_{\lambda \in \alt_x(\mu)}R_{H,T}(\mathcal A, \lambda))^{-1}
\le \max_{\omega \in \triangle_K} \inf_{\lambda \in \alt_{x}(\mu)} H(\lambda)\sum_{k=1}^K \omega_k \KL(\mu_k, \lambda_k)
\: .
\end{align*}
The left hand side is larger than $(\liminf_{T\to +\infty}\sup_{\lambda \in \mathcal D} R_{H,T}(\mathcal A, \lambda))^{-1}$, which is now independent of $\mu$ and $x$. We then take on the right hand side first an infimum over $\mu$, then a liminf over $x$. Doing it in this order leads to the tighter bound (compared to $\inf_\mu \liminf_x$).
\end{proof}

\section{Proofs of results from section~\ref{sec:the_range_of_the_difficulty_ratio}}
\label{sec:proofs_of_results_from_section_sec:the_range_of_the_difficulty_ratio}

For $u, w \in \mathbb{R}^K$, we use the notation $\Vert u \Vert_{\omega} = \sqrt{\sum_{k=1}^K \omega_k u_k^2}$.

\begin{lemma}
For the Gaussian half-space identification problem, where arm $k$ has variance $\sigma_k^2 > 0$, with orthogonal vector $u$ with $\Vert u \cdot \sigma \Vert_1 = 1$, $H_{\mathcal C^{sp}}(\lambda)^{-1} = \frac{1}{2}(\lambda^\top u)^2$.
\end{lemma}

\begin{proof}
We compute $\sup_{\omega \in \triangle_K} \inf_{\nu \in \alt(\lambda)} \sum_{k=1}^K \omega_k \KL(\nu_k, \lambda_k)$ for any $\lambda$.

\begin{align*}
\inf_{\nu \in \alt(\lambda)} \sum_{k=1}^K \omega_k \KL(\nu_k, \lambda_k)
&= \frac{1}{2}\inf_{\nu \in \alt(\lambda)} \sum_{k=1}^K \omega_k \sigma_k^{-2} (\nu_k - \lambda_k)^2
= \frac{1}{2} \frac{(\lambda^\top u)^2}{\Vert u \Vert_{\omega^{-1} \cdot \sigma^{2}}^2}
\end{align*}

\begin{align*}
\sup_{\omega \in \triangle_K} \inf_{\nu \in \alt(\lambda)} \sum_{k=1}^K \omega_k \KL(\nu_k, \lambda_k)
&= \sup_{\omega \in \triangle_K} \frac{1}{2} \frac{(\lambda^\top u)^2}{\Vert u \Vert_{\omega^{-1} \cdot \sigma^{2}}^2}
= \frac{1}{2} (\lambda^\top u)^2
\: .
\end{align*}
\end{proof}

\begin{lemma}\label{lem:half_space}
For Gaussian half-space identification, $\inf_{\mathcal A \in \mathcal C_\infty}\sup_{\lambda \in \mathcal D} R_{H_{\mathcal C^{sp}}, \infty}(\mathcal A, \lambda) \ge 1$.
\end{lemma}
\begin{proof}
For the proof, the vector orthogonal to the hyperplane is $u$ with $\Vert u \cdot \sigma \Vert_1 = 1$.

We show that for all $\nu$,
$\max_{\omega \in \triangle_K} \inf_{\lambda \in \alt(\nu)}
  H_{\mathcal C^{sp}}(\lambda) \sum_{k=1}^K \omega_{k} \KL(\nu_k, \lambda_k) = 1$.
The result then follows from an application of Theorem~\ref{thm:asm_R_LB_limsup}.
\begin{align*}
\max_{\omega \in \triangle_K} \inf_{\lambda \in \alt(\nu)}
  H_{\mathcal C^{sp}}(\lambda) \sum_{k=1}^K \omega_{k} \KL(\nu_k, \lambda_k) 
&= \max_{\omega \in \triangle_K} \inf_{\lambda \in \alt(\nu)}
  \frac{\sum_{k=1}^K \omega_{k} \sigma_k^{-2} (\nu_k - \lambda_k)^2}{(\lambda^\top u)^2}
\\
&= \max_{\omega \in \triangle_K} \inf_{a>0} \frac{1}{a}\inf_{\lambda \in \alt(\nu), (\lambda^\top u)^2 = a}
  \sum_{k=1}^K \omega_{k} \sigma_k^{-2} (\nu_k - \lambda_k)^2
\\
&= \max_{\omega \in \triangle_K} \inf_{a>0} \frac{(\sqrt{a} + \vert u^\top \nu \vert)^2}{a \Vert u \Vert_{\omega^{-1} \cdot \sigma^{2}}^2}
\\
&= \max_{\omega \in \triangle_K} \frac{1}{\Vert u \Vert_{\omega^{-1} \cdot \sigma^{2}}^2}
\\
&= 1
\: .
\end{align*}
\end{proof}

We suppose in the remainder of this section that the distributions of the arms are Gaussian, where arm $k$ has variance $\sigma_k^2 > 0$. The Kullback-Leibler divergence is $(x,y) \mapsto \frac{1}{2 \sigma_k^2}(x - y)^2$.
Suppose that there is a ball $B(\eta, r)$ in the norm $\Vert \cdot \Vert_{\sigma^{-2}}$ with center $\eta \in \mathcal D$ and radius $r > 0$ such that $i^\star$ takes only two values in $B(\eta, r)$, say $i$ and $j$, and the boundary between $B(\eta, r) \cap \{\mu \mid i^\star(\mu) = i\}$ and $B(\eta, r) \cap \{\mu \mid i^\star(\mu) = j\}$ is the restriction of a hyperplane passing through $\eta$. Let $u$ be a vector orthogonal to the hyperplane with $\Vert u \cdot \sigma \Vert_1 = 1$.

\begin{lemma}\label{lem:ball_half_space_difficulty}
For $\mu \in B(\eta, r/(\sqrt{K}+1))$ with $\mu^\top u < \eta^\top u$,
\begin{align*}
\max_{\omega \in \triangle_K} \inf_{\lambda \in \alt(\mu) \cap B(\eta, r)} \sum_{k=1}^K \omega_k (\lambda_k - \mu_k)^2
&= \max_{\omega \in \triangle_K} \inf_{\lambda : (\lambda - \eta)^\top u \ge 0} \sum_{k=1}^K \omega_k (\lambda_k - \mu_k)^2
= ((\mu - \eta)^\top u)^2
\: .
\end{align*}
\end{lemma}
\begin{proof}
Let $\mu \in B(\eta, r/(\sqrt{K}+1))$ be such that $u^\top \mu < u^\top \eta$.
For the full half-space alternative, we have
\begin{align*}
\max_{\omega \in \triangle_K} \inf_{\lambda : (\lambda - \eta)^\top u \ge 0} \sum_{k=1}^K \omega_k \sigma_k^{-2} (\lambda_k - \mu_k)^2 = ((\mu - \eta)^\top u)^2
\end{align*}
Let $\lambda_u(\mu) = \mu - ((\mu - \eta)^\top u) \sigma$. We now prove that that point belongs to the ball $B(\eta, r)$. We will use the fact that $\Vert u \Vert_{\sigma^2}^2 = \sum_{k=1}^K u_k^2 \sigma_k^2 \le \sum_{k=1}^K u_k \sigma_k = 1$ (since $\Vert u \cdot \sigma \Vert_1 = 1)$.
\begin{align*}
\Vert \lambda_{u}(\mu) - \eta \Vert_{\sigma^{-2}}^2
&= \Vert \mu_k - \eta_k - ((\mu - \eta)^\top u) \sigma \Vert_{\sigma^{-2}}^2
\\
&\le \left( \Vert \mu - \eta \Vert_{\sigma^{-2}} + \sqrt{K} \vert(\mu - \eta)^\top u \vert \right)^2
\\
&\le \left( \Vert \mu - \eta \Vert_{\sigma^{-2}} + \sqrt{K} \Vert \mu - \eta \Vert_{\sigma^{-2}} \Vert u \Vert_{\sigma^{2}} \right)^2
\\
&\le (\sqrt{K} + 1)^2 \Vert \mu - \eta \Vert_{\sigma^{-2}}^2
\\
&\le r^2
\: .
\end{align*}
For the problem restricted to the ball,
\begin{align*}
&\max_{\omega \in \triangle_K} \inf_{\lambda \in \alt(\mu) \cap B(\eta, r)} \sum_{k=1}^K \omega_k \sigma_k^{-2} (\lambda_k - \mu_k)^2
\\
&\le \max_{\omega \in \triangle_K} \sum_{k=1}^K \omega_k \sigma_k^{-2} (\lambda_{u,k}(\mu) - \mu_k)^2
= ((\mu - \eta)^\top u)^2 \; ,
\\
\text{and } &\max_{\omega \in \triangle_K} \inf_{\lambda \in \alt(\mu) \cap B(\eta, r)} \sum_{k=1}^K \omega_k \sigma_k^{-2} (\lambda_k - \mu_k)^2
\\
&\ge \max_{\omega \in \triangle_K} \inf_{\lambda : (\lambda - \eta)^\top u \ge 0} \sum_{k=1}^K \omega_k \sigma_k^{-2} (\lambda_k - \mu_k)^2
= ((\mu - \eta)^\top u)^2
\: .
\end{align*}
The last inequality comes from $\alt(\mu) \cap B(\eta, r) \subseteq \{\lambda \mid (\lambda - \eta)^\top u \ge 0\}$. We have proved the equality.
\end{proof}

\begin{lemma}
Let $\delta > 0$, $\varepsilon > 0$, $r' = r / (\sqrt{K} + 1)$ and $r'' = \frac{1}{2} r' \frac{\delta \varepsilon}{(1 + \delta) (1 + \varepsilon)}$.
Let $\mu \in B(\eta, r'')$ with $(\mu - \eta)^\top u > 0$ and let $D_{\varepsilon, \delta}(\mu) = \alt(\mu)\cap B(\eta, r')$. Then
\begin{align*}
\max_{\omega \in \triangle_K} \inf_{\lambda \in D_{\varepsilon, \delta}(\mu)} \sum_{k=1}^K \omega_k \KL(\lambda_k, \mu_k)
\le (1 + \varepsilon) (1 + \delta)^2 \: .
\end{align*}
\end{lemma}

This bound is then used in Theorem~\ref{thm:asm_R_LB_limsup} to get a lower bound on the difficulty ratio. Taking the limit as $\varepsilon \to 0$ and $\delta \to 0$, we prove Theorem~\ref{thm:flat_boundary}.

\begin{proof}
For all $\lambda \in \alt(\mu) \cap B(\eta, r')$, Lemma~\ref{lem:ball_half_space_difficulty} gives $H_{\mathcal C^{sp}}(\lambda) = 2 ((\lambda - \eta)^\top u)^{-2}$.
\begin{align*}
&\max_{\omega \in \triangle_K} \inf_{\lambda \in \alt(\mu) \cap B(\eta, r')}
  H_{\mathcal C^{sp}}(\lambda) \sum_{k=1}^K \omega_{k} \KL(\mu_k, \lambda_k)
\\
&= \max_{\omega \in \triangle_K} \inf_{\lambda \in \alt(\mu) \cap B(\eta, r')}
  \frac{\sum_{k=1}^K \omega_{k} \sigma_k^{-2} (\mu_k - \lambda_k)^2}{((\lambda - \eta)^\top u)^2}
\\
&= \max_{\omega \in \triangle_K} \inf_{\lambda \in \cap B(\eta, r'), (\lambda - \eta)^\top u \le 0}
  \frac{\sum_{k=1}^K \omega_{k} \sigma_k^{-2} (\mu_k - \lambda_k)^2}{((\lambda - \eta)^\top u)^2}
\: .
\end{align*}
If we did not restrict $\lambda$ to the ball $B(\eta, r')$, then that quantity would be equal to 1 as shown in Lemma~\ref{lem:half_space}. We now argue that if $\mu$ is sufficiently close to $\eta$, it approaches 1 even with the restriction to the ball.

For $\omega \in \triangle_K$, let $\omega^\varepsilon \in \triangle_K^0$ be such that $\omega_k^\varepsilon = \frac{\omega_k + \varepsilon}{1 + \varepsilon}$.
\begin{align*}
&\max_{\omega \in \triangle_K} \inf_{\lambda \in \cap B(\eta, r'), (\lambda - \eta)^\top u \le 0}
  \frac{\sum_{k=1}^K \omega_{k} \sigma_k^{-2} (\mu_k - \lambda_k)^2}{((\lambda - \eta)^\top u)^2}
\\
&\le (1 + \varepsilon) \max_{\omega \in \triangle_K} \inf_{\lambda \in \cap B(\eta, r'), (\lambda - \eta)^\top u \le 0}
  \frac{\sum_{k=1}^K \omega_{k}^\varepsilon \sigma_k^{-2} (\mu_k - \lambda_k)^2}{((\lambda - \eta)^\top u)^2}
\end{align*}

Let $x = \frac{1}{2}r' \frac{\varepsilon}{(1 + \delta) (1 + \varepsilon)}$. Let $\lambda_{\omega^\varepsilon}(\mu)$ be the vector with coordinates $\lambda_{\omega^\varepsilon,k}(\mu) = \mu_k - \frac{(\mu - \eta)^\top u + x}{\Vert u \Vert^2_{(\omega^\varepsilon)^{-1} \cdot \sigma^{2}}} \frac{u_k}{\omega_k^\varepsilon} \sigma_k^2$. We show that it belongs to the ball $B(\eta, r')$. This is possible only thanks to the lower bound on any coordinate of $\omega^\varepsilon$, and is the reason for introducing that modification of $\omega$.
\begin{align*}
\Vert \lambda_{\omega^\varepsilon}(\mu) - \eta \Vert_{\sigma^{-2}}
&= \Vert \mu  - \eta - \frac{(\mu - \eta)^\top u + x}{\Vert u \Vert^2_{(\omega^\varepsilon)^{-1} \cdot \sigma^{2}}} (\frac{u_k}{\omega_k^\varepsilon} \sigma_k^2)_{k\in [K]} \Vert_{\sigma^{-2}}
\\
&\le \Vert \mu - \eta \Vert_{\sigma^{-2}} + \frac{(\mu - \eta)^\top u + x}{\Vert u \Vert^2_{(\omega^\varepsilon)^{-1} \cdot \sigma^{2}}} \Vert \frac{u}{\omega^{\varepsilon}} \sigma^2 \Vert_{\sigma^{-2}}
\\
&\le \Vert \mu - \eta \Vert_{\sigma^{-2}} + \frac{\Vert \mu - \eta \Vert_{\sigma^{-2}} + x}{\Vert u \Vert^2_{(\omega^\varepsilon)^{-1} \cdot \sigma^{2}}} \Vert \frac{u}{\omega^{\varepsilon}} \sigma^2 \Vert_{\sigma^{-2}}
\\
&\le \Vert \mu - \eta \Vert_{\sigma^{-2}} + (\Vert \mu - \eta \Vert_{\sigma^{-2}} + x) \Vert \frac{u}{\omega^{\varepsilon}} \sigma^2 \Vert_{\sigma^{-2}}
\\
&\le \Vert \mu - \eta \Vert_{\sigma^{-2}} + (\Vert \mu - \eta \Vert_{\sigma^{-2}} + x)\frac{1 + \varepsilon}{\varepsilon}
\\
&\le r'' + (r'' + x)\frac{1 + \varepsilon}{\varepsilon}
\\
&= x (\delta + (1 + \delta)\frac{1 + \varepsilon}{\varepsilon})
\le 2 x(1 + \delta)\frac{1 + \varepsilon}{\varepsilon} = r'
\: .
\end{align*}
Now since $\lambda_{\omega^\varepsilon}(\mu) \in \alt(\mu) \cap B(\eta, r')$, we get
\begin{align*}
&\max_{\omega \in \triangle_K} \inf_{\lambda \in \cap B(\eta, r'), (\lambda - \eta)^\top u \le 0}
  \frac{\sum_{k=1}^K \omega_{k} \sigma_k^{-2} (\mu_k - \lambda_k)^2}{((\lambda - \eta)^\top u)^2}
\\
&\le (1 + \varepsilon) \max_{\omega \in \triangle_K} \inf_{\lambda \in \cap B(\eta, r'), (\lambda - \eta)^\top u \le 0}
  \frac{\sum_{k=1}^K \omega_{k}^\varepsilon \sigma_k^{-2} (\mu_k - \lambda_k)^2}{((\lambda - \eta)^\top u)^2}
\\
&\le (1 + \varepsilon) \max_{\omega \in \triangle_K} 
  \frac{\sum_{k=1}^K \omega_{k}^\varepsilon \sigma_k^{-2} (\mu_k - \lambda_{\omega^\varepsilon,k}(\mu))^2}{((\lambda_{\omega^\varepsilon}(\mu) - \eta)^\top u)^2}
\: .
\end{align*}
We can compute explicitly both terms in the ratio:
\begin{align*}
\sum_{k=1}^K \omega_{k}^\varepsilon \sigma_k^{-2} (\mu_k - \lambda_{\omega^\varepsilon,k}(\mu))^2
&= \frac{((\mu - \eta)^\top u + x)^2}{\Vert u \Vert^2_{(\omega^\varepsilon)^{-1} \cdot \sigma^{2}}}
\: ,
&
(\lambda_{\omega^\varepsilon}(\mu) - \eta)^\top u
&= -x
\: .
\end{align*}
Finally,
\begin{align*}
&\max_{\omega \in \triangle_K} \inf_{\lambda \in \cap B(\eta, r'), (\lambda - \eta)^\top u \le 0}
  \frac{\sum_{k=1}^K \omega_{k} (\mu_k - \lambda_k)^2}{((\lambda - \eta)^\top u)^2}
\\
&\le (1 + \varepsilon) \max_{\omega \in \triangle_K} \inf_{\lambda \in \cap B(\eta, r'), (\lambda - \eta)^\top u \le 0}
  \frac{\sum_{k=1}^K \omega_{k}^\varepsilon (\mu_k - \lambda_k)^2}{((\lambda - \eta)^\top u)^2}
\\
&\le (1 + \varepsilon) \max_{\omega \in \triangle_K}
  \frac{((\mu - \eta)^\top u + x)^2}{x^2 \Vert u \Vert^2_{(\omega^\varepsilon)^{-1} \cdot \sigma^{2}}}
\\
&\le (1 + \varepsilon) (\frac{(\mu - \eta)^\top u}{x} + 1)^2
\\
&\le (1 + \varepsilon) (\frac{r''}{x} + 1)^2
\\
&= (1 + \varepsilon) (1 + \delta)^2
\: .
\end{align*}
\end{proof}

\section{Proofs of results from Section~\ref{sec:no_complexity_in_best_arm_identification}}
\label{sec:proofs_relative_to_section_sec:no_complexity_in_best_arm_identification}

\subsection{Gaussian bandits}
\label{sub:gaussian_bai_proof}

\begin{proof}[of Theorem~\ref{thm:gaussian_bai_no_complexity}]
First, since $\mathcal C^{sp} \subseteq \mathcal C$, for any algorithm $\mathcal A$ and $\mu \in \mathcal D$, $R_{H_{\mathcal C}, T}(\mathcal A, \mu) \ge R_{H_{\mathcal C^{sp}}, T}(\mathcal A, \mu)$. It suffices to give a lower bound for $H_{\mathcal C^{sp}}$.

Let $H_{\Delta}(\mu) = \frac{2}{\min_{k : \Delta_k>0} \Delta_k^2} + \sum_{k:\Delta_k>0} \frac{2}{\Delta_k^2}$. It was shown in \citep{garivier2016optimal} that for all $\mu \in \mathcal D$, this function satisfies the inequalities $H_\Delta(\mu) \le H_{C^{sp}}(\mu) \le 2 H_\Delta(\mu)$ .
Thus $R_{H_{\mathcal C}, T}(\mathcal A, \mu) \ge R_{H_\Delta, T}(\mathcal A, \mu) / 2$. From this point on, we use a construction similar to the one used in \citep{carpentier2016tight} to prove a lower bound on the ratio to $H_\Delta$ for Bernoulli bandits.
We define a Gaussian problem $\mu$ by $\mu_1 = 0$ (or any arbitrary value) and $\mu_k = \mu_1 - k \Delta$ for all $k \in \{2, \ldots, K\}$ and some arbitrary $\Delta > 0$.
We apply Corollary~\ref{cor:corner_lb} to $\mu$ and $\lambda^{(2)}, \ldots, \lambda^{(K)}$ where each $\lambda^{(j)}$ is identical to $\mu$ except that $\lambda^{(j)}_j = \mu_1 + (\mu_1 - \mu_j)$.
\begin{align*}
\sup_{j \in \{2, \ldots, K\}} \limsup_{T\to +\infty} R_{H_\Delta,T}(\mathcal A, \lambda^{(j)})
\ge \sum_{j=2}^K \frac{1}{H_\Delta(\lambda^{(j)}) \KL(\mu_j, \lambda_j^{(j)})}
= \sum_{j=2}^K \frac{1}{\frac{(\lambda_j^{(j)} - \mu_j)^2}{(\lambda_j^{(j)} - \mu_1)^2} + \sum_{k\ne j} \frac{(\lambda_j^{(j)} - \mu_j)^2}{(\lambda_j^{(j)} - \mu_k)^2}}
\: .
\end{align*}
For our specific choice of $\lambda^{(j)}$,
\begin{align*}
\frac{(\lambda_j^{(j)} - \mu_j)^2}{(\lambda_j^{(j)} - \mu_1)^2} + \sum_{k \ne j} \frac{(\lambda_j^{(j)} - \mu_j)^2}{(\lambda_j^{(j)} - \mu_k)^2}
&\le 4 + 4 \sum_{k \ne j} \frac{(\mu_1 - \mu_j)^2}{(\mu_1 - \mu_j)^2 + (\mu_1 - \mu_k)^2}
\le 4 j + 4\sum_{k > j} \frac{(\mu_1 - \mu_j)^2}{(\mu_1 - \mu_k)^2}
\: .
\end{align*}
We now use that $\mu_k = \mu_1 - k \Delta$.
\begin{align*}
\frac{(\lambda_j^{(j)} - \mu_j)^2}{(\lambda_j^{(j)} - \mu_1)^2} + \sum_{k \ne j} \frac{(\lambda_j^{(j)} - \mu_j)^2}{(\lambda_j^{(j)} - \mu_k)^2}
&\le 4 j + 4 j^2 \sum_{k > j} \frac{1}{k^2}
\le 4 j + 4 j^2 \frac{1}{j}
\le 8 j \: .
\end{align*}
We finally have the lower bound
\begin{align*}
\sup_{j \in \{2, \ldots, K\}} \limsup_{T\to +\infty} R_{H_\Delta,T}(\mathcal A, \lambda^{(j)})
\ge \frac{1}{8}\sum_{j=2}^K \frac{1}{j}
\ge \frac{1}{8}(\log(K+1) - \log 2)
\ge \frac{3}{40} \log K
\: .
\end{align*}
\end{proof}

\subsection{Bernoulli bandits}
\label{sub:bernoulli_bai_proof}

We consider the best arm identification task in bandits with two arms, both in the same exponential family with one parameter. Two distributions in that family with means $\mu_1, \mu_2$ correspond to some natural parameters $\xi_1, \xi_2$ and the Kullback-Leibler divergence can be written
\begin{align*}
\KL(\mu_1, \mu_2) = d(\xi_2, \xi_1) = \phi(\xi_2) - \phi(\xi_1) - (\xi_2 - \xi_1) \phi'(\xi_1) \: ,
\end{align*}
where $\phi : \mathbb{R} \to \mathbb{R}$ is a convex function specific to the exponential family and $d$ is its Bregman divergence. The mean parameter $\mu_1$ and the corresponding natural parameter $\xi_1$ are related by the equation $\phi'(\xi_1) = \mu_1$ (or $\xi_1 = \phi'^{-1}(\mu_1)$ since $\phi'$ is invertible).
In that setting, we want to compute
\begin{align*}
(H_{\mathcal C^{sp}}(\mu))^{-1}
&= \max_{\omega \in \triangle_K} \inf_{\lambda \in \alt(\mu)} \sum_{k=1}^K \omega_k \KL(\lambda_k, \mu_k)
\\
&= \max_{\omega \in \triangle_2} \inf_{x} (\omega_1 \KL(x, \mu_1) + \omega_2 \KL(x, \mu_2))
\: .
\end{align*}

\begin{lemma}\label{lem:exp_fam_bai}
In the one-parameter exponential family setting described above,
\begin{align*}
(H_{\mathcal C^{sp}}(\mu))^{-1} = \KL\left(\frac{\phi(\xi_1) - \phi(\xi_2)}{\xi_1 - \xi_2}, \mu_1\right) \: .
\end{align*}
The infimum in the definition of the difficulty is attained for any $\omega$ at $x(\omega) = \phi'(\omega_1 \xi_1 + \omega_2 \xi_2)$. The maximum over the simplex is attained at $\omega^\star$ such that $\omega^\star_1 = \frac{\phi'^{-1}(x^\star) - \xi_2}{\xi_1 - \xi_2}$, with $x^\star = x(\omega^\star) = \frac{\phi(\xi_1) - \phi(\xi_2)}{\xi_1 - \xi_2}$.
\end{lemma}
We can also rewrite $\frac{\phi(\xi_1) - \phi(\xi_2)}{\xi_1 - \xi_2} = \mu_2 + \frac{\KL(\mu_2, \mu_1)}{\xi_1 - \xi_2} = \mu_1 - \frac{\KL(\mu_1, \mu_2)}{\xi_1 - \xi_2}$.

\begin{proof}
We parametrize by the natural parameters:
\begin{align*}
\inf_{x} (\omega_1 \KL(x, \mu_1) + \omega_2 \KL(x, \mu_2))
= \inf_{y} (\omega_1 d(\xi_1, y) + \omega_2 d(\xi_2, y))
\end{align*}

The optimality condition for $y$ is $\omega_1 \frac{\partial}{\partial y}d(\xi_1, y) + \omega_2 \frac{\partial}{\partial y}d(\xi_2, y) = 0$. That derivative is $\frac{\partial}{\partial y}d(x, y) = - (x - y)\phi''(y)$. We obtain
\begin{align*}
&\omega_1 (\xi_1 - y)\phi''(y) + \omega_2 (\xi_2 - y)\phi''(y) = 0
\\
\implies \quad &y = \omega_1 \xi_1 + \omega_2 \xi_2 \: .
\end{align*}

We note for later the property
\begin{align}\label{eq:exp_fam_minimizer_grad}
\omega_1 \frac{\partial}{\partial y}d(\xi_1, y) + \omega_2 \frac{\partial}{\partial y}d(\xi_2, y) = 0 \quad \text{at } y = \omega_1 \xi_1 + \omega_2 \xi_2 \: .
\end{align}

We now want to compute
\begin{align*}
\max_{\omega \in \triangle_2} (\omega_1 d(\xi_1, \omega_1 \xi_1 + \omega_2 \xi_2) + \omega_2 d(\xi_2, \omega_1 \xi_1 + \omega_2 \xi_2))
\\
=\max_{\omega_1 \in [0,1]} (\omega_1 d(\xi_1, \omega_1 \xi_1 + (1 - \omega_1) \xi_2) + (1 - \omega_1) d(\xi_2, \omega_1 \xi_1 + (1 - \omega_1) \xi_2))
\end{align*}

At the optimal value for $\omega$ the gradient is zero:
\begin{align*}
&d(\xi_1, \omega_1 \xi_1 + (1 - \omega_1) \xi_2) - d(\xi_2, \omega_1 \xi_1 + (1 - \omega_1) \xi_2)
  + \omega_1 \frac{\partial}{\partial y} d(\xi_1, \omega_1 \xi_1 + (1 - \omega_1) \xi_2) (\xi_1 - \xi_2)
  \\&+ (1 - \omega_1) \frac{\partial}{\partial y} d(\xi_2, \omega_1 \xi_1 + (1 - \omega_1) \xi_2) (\xi_1 - \xi_2)
= 0
\end{align*}
We use Equation~\eqref{eq:exp_fam_minimizer_grad} to get that $\omega_2 \frac{\partial}{\partial y} d(\xi_2, \omega_1 \xi_1 + (1 - \omega_1) \xi_2) = - \omega_1 \frac{\partial}{\partial y} d(\xi_1, \omega_1 \xi_1 + (1 - \omega_1) \xi_2)$. We simplify the equation to
\begin{align*}
d(\xi_1, \omega_1 \xi_1 + (1 - \omega_1) \xi_2)
= d(\xi_2, \omega_1 \xi_1 + (1 - \omega_1) \xi_2)
\end{align*}
We expand the Bregman divergence.
\begin{align*}
&\phi(\xi_1) - \phi(y) - (\xi_1 - y) \phi'(y) - \phi(\xi_2) + \phi(y) + (\xi_2 - y) \phi'(y) = 0
\\
\implies &\phi'(y) = \frac{\phi(\xi_1) - \phi(\xi_2)}{\xi_1 - \xi_2}
\end{align*}

Solving this equation for $y$ also gives the value of $\omega$ thanks to $y = \omega_1 \xi_1 + (1 - \omega_1) \xi_2$. We get $\omega_1 = \frac{y - \xi_2}{\xi_1 - \xi_2}$, and $y$ is given by the equation above.
The value of the objective is then
\begin{align*}
\max_{\omega \in \triangle_2} \inf_{x} (\omega_1 \KL(x, \mu_1) + \omega_2 \KL(x, \mu_2))
= d(\xi_1, y)
\\
\text{where }
y = \phi'^{-1}\left(\frac{\phi(\xi_1) - \phi(\xi_2)}{\xi_1 - \xi_2}\right)
\: .
\end{align*}
But we can simplify this further since $d(\xi_1, y) = \KL(\phi'(y), \mu_1)$ (also equal to $\KL(\phi'(y), \mu_2)$).
\begin{align*}
\max_{\omega \in \triangle_2} \inf_{x} (\omega_1 \KL(x, \mu_1) + \omega_2 \KL(x, \mu_2))
= \KL(\frac{\phi(\xi_1) - \phi(\xi_2)}{\xi_1 - \xi_2}, \mu_1)
\: .
\end{align*}
\end{proof}

\begin{lemma}
If the distributions with parameters $\mu_1$ and $\mu_2$ are $\sigma^2$-sub-Gaussian, then
\begin{align*}
(H_{\mathcal C^{sp}}(\mu))^{-1}
\ge \frac{1}{2 \sigma^2(\xi_1 - \xi_2)^2}\max\{\KL(\mu_1, \mu_2)^2, \KL(\mu_2, \mu_1)^2\}
\: .
\end{align*}
For an exponential family of Gaussians with same variance $\sigma^2$ there is equality, and the two terms of the maximum are equal.
\end{lemma}

\paragraph{Gaussian case}
For Gaussian distributions, the functions used above are
\begin{itemize}[noitemsep]
  \item $\phi(a) = \frac{1}{2} \sigma^2 a^2$ with $\phi'(a) = a \sigma^2$, $\phi'^{-1}(x) = x / \sigma^2$, $\phi(\phi'^{-1}(x)) = \frac{1}{2 \sigma^2} x^2$
  \item $d(a, b) = \sigma^2(\frac{1}{2} a^2 - \frac{1}{2} b^2 - (a - b) b) = \frac{1}{2}\sigma^2(a - b)^2$
  \item $\KL(x, y) = \frac{1}{2 \sigma^2}(x - y)^2$.
\end{itemize}

Using these values in Lemma~\ref{lem:exp_fam_bai} gives a static proportions difficulty equal to the inverse of $\frac{1}{8 \sigma^2}(\mu_1 - \mu_2)^2$.

\paragraph{Bernoulli case}
For Bernoulli distributions, the functions used above are
\begin{itemize}[noitemsep]
  \item $\phi(a) = \log (1 + e^a)$ with $\phi'(a) = \frac{e^a}{1 + e^a}$, $\phi'^{-1}(x) = \log \frac{x}{1 - x}$, $\phi(\phi'^{-1}(x)) = - \log (1 - x)$
  \item $d(a, b) = \log(1 + e^a) - \log(1 + e^b) - (a - b) \frac{e^b}{1 + e^b}$
  \item $\KL(x, y) = x \log \frac{x}{y} + (1-x) \log \frac{1-x}{1-y}$.
\end{itemize}

Using these values in Lemma~\ref{lem:exp_fam_bai} proves Lemma~\ref{cor:bernoulli_bai_H}.
\begin{align*}
\max_{\omega \in \triangle_2} \inf_{x} (\omega_1 \KL(x, \mu_1) + \omega_2 \KL(x, \mu_2))
= \KL \left( \frac{\log\frac{1 - \mu_2}{1 - \mu_1}}{\log \frac{\mu_1(1 - \mu_2)}{(1 - \mu_1)\mu_2}}, \mu_1 \right)
\: .
\end{align*}

We gather now a few limits, which will be useful in the proof of Theorem~\ref{thm:bernoulli_bai_no_complexity}. These results use the explicit formulas for $H_{\mathcal C^{sp}}$ derived above.
\begin{align*}
\lim_{x \to 0} H_{\mathcal C^{sp}}((x, 1/2))
&= 1/\log 2
\: , \\
\lim_{x \to 0} \KL(x, 1/2)
&= \log 2
\: , \\
\lim_{x \to 0, y \to 0, x/y \to 1} \frac{1}{H_{\mathcal C^{sp}}((y/2, y)) \KL(x, y/2)}
&= \frac{1 -  \frac{1}{2 \log 2} - \frac{\log(2 \log 2)}{2 \log 2}}{\log 2 - 1/2} \approx 0.22
\: .
\end{align*}

\begin{proof}[of Theorem~\ref{thm:bernoulli_bai_no_complexity}]
For $x \in (0,1/2)$, let $\mu(x) = (x(1+x), x)$, $\lambda^{(1)}(x) = (x/2, x)$, $\lambda^{(2)}(x) = (x(1+x), 1/2)$. Then Corollary~\ref{cor:corner_lb} gives
\begin{align*}
&\sup_{j \in [2]} R_{H_{\mathcal C^{sp}}, \infty}(\mathcal A, \lambda^{(j)}(x))
\\
&\ge \frac{1}{H_{\mathcal C^{sp}}(\lambda^{(1)}(x))\KL(\mu_1(x), \lambda^{(1)}_1(x))} + \frac{1}{H_{\mathcal C^{sp}}(\lambda^{(2)}(x))\KL(\mu_2(x), \lambda^{(2)}_2(x))}
\\
&= \frac{1}{H_{\mathcal C^{sp}}((x/2, x)) \KL(x(1+x), x/2)} + \frac{1}{H_{\mathcal C^{sp}}((x(1+x), 1/2))\KL(x, 1/2)}
\: .
\end{align*}
The limit of the quantity on the right when $x \to 0$ is strictly greater than 1 (it is approximately 1.22).
\end{proof}

\end{document}